\documentclass{article}

\usepackage{arxiv}

\usepackage[utf8]{inputenc} 
\usepackage[T1]{fontenc}    
\usepackage{hyperref}       
\usepackage{url}            
\usepackage{booktabs}       
\usepackage{amsfonts}       
\usepackage{nicefrac}       
\usepackage{microtype}      
\usepackage{lipsum}		
\usepackage{graphicx}
\usepackage[numbers]{natbib}\usepackage{doi}
\usepackage{color}
\usepackage{amsmath,amssymb,amsthm}
\newtheorem{lemma}{Lemma}

\newtheorem{assum}{Assumption}
\newtheorem{theorem}{Theorem}

\newtheorem{remark}{Remark}
\newtheorem{problem}{Problem}

\newtheorem{definition}{Definition}
\newtheorem{example}{Example}

\newcommand{\IEEEPARstart}[2]{#1#2}
\newcommand{\uqed}{\hfill\ensuremath{\square}}
\newcommand{\lb}{[}
\newcommand{\rb}{]}

\title{From Non-Rigid to Rigid: Safe Acquisition of Rigid Communication Graphs under Limited Sensing}

\author{ \href{https://orcid.org/0000-0003-0212-0762}{\includegraphics[scale=0.06]{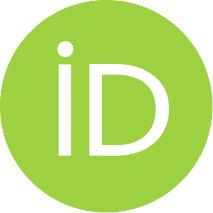}\hspace{1mm}S. Saharsh} \\
	Department of Cyber Physical Systems\\
	Indian Institute of Science\\
	Bengaluru, 560012, India \\
	\texttt{saharsh2021@iisc.ac.in} \\
        \And
	  Vedhas Talnikar \\
	Department of Cyber Physical Systems\\
	Indian Institute of Science\\
	Bengaluru, 560012, India\\
	\texttt{talnikarvedhas@gmail.com} \\
        \And
	Pushpak Jagtap \\
	Department of Cyber Physical Systems\\
	Indian Institute of Science\\
	Bengaluru, 560012, India\\
	\texttt{pushpak@iisc.ac.in} \\
}

\renewcommand{\shorttitle}{\textit{arXiv} Template}

\begin{document}
\maketitle

\begin{abstract}
Communication graph rigidity is a fundamental requirement in many multi-robot formation control approaches. However, ensuring and maintaining a rigid communication topology becomes challenging in practice due to limited sensing ranges and dynamic operating conditions. This paper provides a method for achieving an inter-robot collision-free, rigid time-varying communication graph, where communication links are established or broken according to limited sensing ranges, without assuming an initial rigid graph. In addition, the proposed approach guarantees the realization of a rigid graph for heterogeneous nonlinear multi-robot systems. A computationally lean, distributed quadratic optimization-based controller is developed for a leader–follower architecture, acquiring rigidity based on hierarchical second-order consensus among robots. Follower agents do not require global absolute positions of any agent, including their own. The proposed method is validated through both simulations and hardware experiments in a motion-capture environment, demonstrating reliable performance under the limited sensing capabilities of individual robots. For paper walkthrough: \url{https://youtu.be/PyHoO0E9oLg}.
\end{abstract}


\section{Introduction}
\label{sec:introduction}
\IEEEPARstart{C}{ollective} motion in nature, such as in flocks of birds, schools of fish, and colonies of ants, demonstrates how individuals can coordinate their movement while preserving spatial organization within a group. Inspired by such biological systems, multi-robot systems (MRS) have gained significant attention in civilian and military applications, including collaborative transportation, environmental monitoring, and border surveillance \cite{baillieul2007control}. In these applications, multiple autonomous robots cooperate to achieve complex tasks more efficiently than a single robot operating alone. A key requirement in such cooperative missions is the ability to move the group as a rigid body or maintain a desired unique formation while ensuring safe operation in dynamic environments \cite{bullo2009distributed}.

Formation maintenance in MRS \cite{liu2024survey} fundamentally depends on the underlying interaction topology among robots. Rigid graphs are essential to achieve unique formation shapes while using the least number of inter-robot connections \cite{anderson2003operations}. This reduction in communication links directly reduces communication and sensing requirements, thereby improving scalability and energy efficiency in distributed multi-robot systems \cite{anderson2018rigid, eren2004rigid}. However, the acquisition and maintenance of rigid communication remains the central problem in distributed formation control, and this can be illustrated through a simple example.
\begin{example}
\label{CBRS: ex/treenr}
Consider four robots connected in a tree (chain) structure 
$1\!-\!2\!-\!3\!-\!4$. Although the graph is connected with 
$n-1=3$ edges, it does not preserve a unique formation shape through the tree edges alone, i.e., it is not rigid in $\mathbb{R}^2$. The formation 
can bend into a ring while preserving all existing edge distances, causing 
unconstrained distances, such as between agents $1$ and $4$, to vary. In general, rigidity in 2D requires at least $2n-3$ edges, carefully distributed across all agents, for $n=4$, this means at least $5$ edges are needed. \uqed
\end{example}
Despite significant progress in formation control, most existing works \cite{zhao2016localizability, pampatwar2021planar, rayabagi2024formation} assume that a rigid communication graph is already available and maintained throughout the mission. However, in practical scenarios, agents are often subject to limited sensing ranges \cite{bhatia2025decentralized} and restricted fields of view \cite{dias2016distributed}, making rigid graph acquisition itself a challenging problem.
Therefore, this work aims to develop a systematic method for MRS, to search and create requisite edges, distributed across agents, ensuring rigidity under sensing-range limitations in the 2D plane.

Alongside rigidity acquisition, inter-agent collision avoidance remains a critical challenge in practical MRS deployment. Since agents operate in close proximity while maneuvering cooperatively, safety constraints must be enforced without destroying the desired formation structure or rigid interaction topology. To address this issue, distributed safe control schemes based on Control Barrier Functions (CBFs) \cite{wang2017safety, tan2022distributed} integrated with quadratic programming (QP) have been widely adopted for real-time collision avoidance and constraint satisfaction. Collision cone-based CBFs, such as C3BF \cite{tayal2024collision}, guarantee safety only for obstacles moving at constant velocity, limiting inter-agent safety in dynamic scenarios.

Furthermore, achieving collision-free motion during rigid graph formation under a limited sensing range can result in deadlocks, where agents remain indefinitely engaged in pursuing rigidity acquisition while repeatedly stopping due to safety constraints imposed by neighboring agents. Only a limited number of studies investigate rigid graph requirements \cite{frank2018bearing} under field-of-view constraints, and these generally assume arbitrarily large sensing ranges. However, graph rigidity or connectivity maintenance \cite{zelazo2015decentralized, zelazo2012rigidity} is widely studied, but achieving rigidity along with inter-agent safety from a non-rigid graph is not explored. Therefore, developing distributed methods for rigid graph acquisition with collision avoidance under limited sensing range remains an important problem in multi-agent formation control.
This work addresses the problem with three main contributions in the following.
\begin{itemize}
    \item A novel CBF-QP-like optimization framework for acquiring rigidity and ensuring inter-agent safety within a \emph{heterogeneous} multi-agent system with limited sensing range,
    \item Formal guaranty for forming a rigid graph with dynamic communication links in a leader-follower framework, including agents with different sensing ranges, using \emph{only} relative positions and velocities of neighboring agents, without agent's own global position and velocity,
    \item Guaranteed collision-free motion for accelerating nonlinear agents, extending \cite{tayal2024collision} beyond constant-velocity obstacles, with control parameters derived from explicit stability conditions instead of trial-and-error tuning.
\end{itemize}
The \emph{key novelty} of the proposed framework lies in the indirect shaping of the time-varying part of the graph through control actions defined exclusively on the time-invariant subgraph, where global reference frame information is not required for the controlled agents, suitable for GPS-denied environments. Specifically, using the relative state information of the neighboring agents, the evolution of agent states over the time-invariant subgraph induces a structured reconfiguration of edges in the time-varying subgraph, resulting in the emergence of a rigid subgraph. By construction, the minimal rigidity of the subgraph guarantees the rigidity of the overall communication graph.

\section{Preliminaries and Problem Formulation}
\label{CBRS: sec/prelim and prob defn}
\subsection{Notations}
 Scalars are denoted as $x \in \mathbb{R}$, and vectors as $\boldsymbol{x} \in \mathbb{R}^n$ (bold letter). For $a, b \in \mathbb{N}$ and $a < b$, we use the notation $\lb a,\ b \rb_{\mathbb{N}}$ to represent the closed interval in $\mathbb{N}$. For $\boldsymbol{x} \in \mathbb{R}^n$,  $\|\boldsymbol{x}\|$ denotes its Euclidean norm. Let $\boldsymbol{1}:= [1,\ 1, \cdots, 1]^{\top}$, $\boldsymbol{0}:= [0,\ 0, \cdots, 0]^{\top}$, $\mathbf{e}_{n} := [1, 0, \cdots, 0]^{\top}$ of size $n$.
 Given a set $\mathcal{C}$, $\partial \mathcal{C}$ denotes its boundary. The cardinality of the set $A$ is denoted by $|A|$. The function $\kappa: \mathbb{R} \rightarrow \mathbb{R}$ is of extended class $\mathcal{K}_{e}$ if it is continuous, $\kappa(0)=0$, and strictly increasing. 
The closed ball $\mathbb{B}(\boldsymbol{c},\lambda) \subseteq \mathbb{R}^2$ is defined as $\mathbb{B}(\boldsymbol{c}, \lambda) := \{ \boldsymbol{x} \in \mathbb{R}^2 : \|\boldsymbol{x} - \boldsymbol{c}\| \leq \lambda\}$, where $\boldsymbol{c}\in \mathbb{R}^2$ is the center and $\lambda \in \mathbb{R}^+$ is the radius. All other notation follows standard mathematical conventions.

\subsection{System Description}
We consider a heterogeneous multi-agent system (MAS) consisting of $n$ agents, where each agent is indexed by $i \in \lb 1,\ n \rb_{\mathbb{N}}$. Each $i^{th}$ agent (modeled as a planar robot) is represented as:
\begin{align}
        \dot{\boldsymbol{p}}_i = \boldsymbol{v}_i, 
        \dot{\boldsymbol{v}}_i = f_{v_i}(\boldsymbol{p}_i,\boldsymbol{v}_i) + h_{v_i}(\boldsymbol{p}_i,\boldsymbol{v}_i, \delta_i)\boldsymbol{u}_i, 
        \label{CBRS: eqn/sys}
\end{align}
    where $\boldsymbol{p}_i,\boldsymbol{v}_i,\boldsymbol{u}_i\in \mathbb{R}^2$ denote the position in the 2D plane, velocity, and input vectors, respectively, and $\delta_i \in (-\pi,\pi)$ denotes the measurable heading angle. The maps $f_{v_i} : \mathbb{R}^2 \times \mathbb{R}^2 \rightarrow \mathbb{R}^2, f_{v_i}(\boldsymbol{0}, \boldsymbol{0}) = \boldsymbol{0},$ and $h_{v_i}: \mathbb{R}^2 \times \mathbb{R}^2 \times (-\pi, \pi)\rightarrow \mathbb{R}^{2 \times 2}, h_{v_i}(\boldsymbol{0}, \boldsymbol{0}, 0) = I_2,$ are globally Lipschitz continuous functions with Euclidean norm's Lipschitz constants $\Gamma_{if}, \Gamma_{ih},$ respectively. The system is controllable and $\det(h_{v_i}(\boldsymbol{p}_i,\boldsymbol{v}_i, \delta_i)) \neq 0, \forall \boldsymbol{p}_i,\boldsymbol{v}_i \in \mathbb{R}^2, \delta_i \in (-\pi, \pi)$. For notational simplicity, we use the shorthand notation $h_{v_i}(\boldsymbol{p}_i,\boldsymbol{v}_i) := h_{v_i}(\boldsymbol{p}_i,\boldsymbol{v}_i, \delta_i)$, since each agent knows its heading angle $\delta_i$, but not its global state variables $\boldsymbol{p}_i,\boldsymbol{v}_i$. The input gain matrix $h_{v_i}$ is generally non-symmetric due to the rotational effects induced by the heading angle $\delta_i \in (-\pi,\pi)$. So, we assume that
    \begin{assum}
         \label{CBRS: assum/A2}
         For each agent $i$, there exists a matrix $g_i(\delta_i) \in \mathbb{R}^{2 \times 2}$ such that the matrix $h_{v_i}(\boldsymbol{p}_i,\boldsymbol{v}_i)g_i(\delta_i)$ is symmetric and positive definite. Also, we have $\boldsymbol{x}^{\top}h_{v_i}(\boldsymbol{p}_i,\boldsymbol{v}_i)g_i(\delta_i)\boldsymbol{x} \leq \mu_{\max}(h_{v_i}(\boldsymbol{p}_i,\boldsymbol{v}_i)g_i(\delta_i))\|\boldsymbol{x}\|^2 \leq\Gamma_{ig}(\|\boldsymbol{p}_i\| + \|\boldsymbol{v}_i\|)\|\boldsymbol{x}\|^2, \Gamma_{ig} \in \mathbb{R}^{+}$, for any $\boldsymbol{x} \in \mathbb{R}^2$, where $\mu_{\max}(.)$ is the maximum eigenvalue and $\det(h_{v_i}(\boldsymbol{p}_i,\boldsymbol{v}_i)g_i(\delta_i)) \geq \Gamma_i, \Gamma_i \in \mathbb{R}^{+}$.
    \end{assum}
    \begin{assum}
        \label{CBRS: assum/A4}
    Given $\mathcal{V}$ as the set of agents, we assume that there exist constants
    $\bar{\Delta}, \bar{\epsilon} \in \mathbb{R}_0^+$ such that, for every pair of
    agents $i,k \in \mathcal{V}$ and for all $(\boldsymbol{p},\boldsymbol{v})
    \in \mathbb{R}^2 \times \mathbb{R}^2$,
    \begin{align}
        \big\|f_{v_i}(\boldsymbol{p},\boldsymbol{v})
            - f_{v_k}(\boldsymbol{p},\boldsymbol{v})\big\|
            &\leq \bar{\Delta}, \label{eq:flipsch_c}\\
        \big\|h_{v_i}(\boldsymbol{p},\boldsymbol{v})
            - h_{v_k}(\boldsymbol{p},\boldsymbol{v})\big\|
            &\leq \bar{\epsilon}. \label{eq:hlipsch_c}
    \end{align}
    Also, for any pair of agents with the same system dynamics, i.e., $f_{v_i} = f_{v_k}, h_{v_i} = h_{v_k}$, $\bar{\Delta} = \bar{\epsilon} = 0$.
    \end{assum}
\begin{remark}
    A broad class of planar robotic systems can be modeled as \eqref{CBRS: eqn/sys}, for example, a simple double-integrator, Ackermann-drive vehicles, differential-drive robots as well as aerial vehicles such as drones and fixed-wing aircraft operating in altitude hold mode. For a detailed derivation, especially for a system with orientation state, refer to \cite[Section II B]{sawarkar2026sliding}. The input gain matrix for a simple double-integrator $h_{v_i}(\boldsymbol{p}_i,\boldsymbol{v}_i)$ is an identity matrix with $g_i(\delta_i) = I_2$ as it does not have a heading angle. For an Ackermann-drive with small slip or differential-drives, we have $h_{v_i}(\boldsymbol{p}_i,\boldsymbol{v}_i,\delta_i) = { \left[ \begin{matrix} \cos(\delta_i) & -\|\boldsymbol{v}_i\| \sin(\delta_i) \\ \sin(\delta_i) & \|\boldsymbol{v}_i\| \cos(\delta_i) \end{matrix} \right]}$ with $g_i(\delta_i) = { \left[ \begin{matrix} \cos(\delta_i) & - \sin(\delta_i) \\ \sin(\delta_i) &  \cos(\delta_i) \end{matrix} \right]}^{\top}$, which makes the matrix $h_{v_i}g_i$ symmetric and positive definite, thereby satisfying the Assumption \ref{CBRS: assum/A2}. Also, we can observe that $\det(h_{v_i}g_i) \geq v_{\min} = \Gamma_i$, where $v_{\min}$ is the minimum positive speed of the vehicle.
\end{remark}

\subsection{Interleaved Layered Directed Acyclic Graph Architecture}

Consider a group of $n \geq 2$ agents whose communication network is represented as a time-varying directed graph $\mathcal{G}(t) := (\mathcal{V}, \mathcal{E}(t), \mathcal{A}(t))$. Here, $\mathcal{V} = \lb 1, n\rb_{\mathbb{N}}$ is the set of vertices, $\mathcal{E}(t) \subseteq \mathcal{V} \times \mathcal{V}$ is the set of edges at time $t \in \mathbb{R}_{0}^{+}$, and $\mathcal{A}(t) = [a_{ij}(t)]_{n \times n}$ is the weighted adjacency matrix at time $t \in \mathbb{R}_{0}^{+}$, where $a_{ij}(t) \in \mathbb{R}_{0}^{+}, a_{ii}(t) = 0$. A directed edge of $\mathcal{G}(t)$ is denoted by $(i,j)$, where $j$ is the parent of $i$ and $i$ is the child of $j$. The parent and grandparent of the vertex $i$ are denoted by $pr(i)$ and $pr^{(2)}(i)$, respectively. Each agent $i \in \mathcal{V}$ has a sensing range $\lambda_i \in \mathbb{R}^+$. The edge condition $(i,j) \in \mathcal{E}(t) \iff a_{ij}(t) > 0 \iff \boldsymbol{p}_j(t) \in \mathbb{B}(\boldsymbol{p}_i(t), \lambda_i)$ signifies that agent $i$ can receive information from agent $j$. Otherwise, $(i,j) \not \in \mathcal{E}(t) \iff a_{ij}(t) = 0$. Since sensing ranges may differ across agents, $(i,j) \in \mathcal{E}(t)$ does \emph{not} imply $(j,i) \in \mathcal{E}(t)$, favoring the directed graph topology. We consider a leader-follower framework, where the set $\mathcal{V}$ is divided into a leader set $\mathcal{V}_L \subseteq \mathcal{V}$ with $|\mathcal{V}_L| = n_l \geq 2$ and a follower set $\mathcal{F} := \mathcal{V} \setminus \mathcal{V}_L$. Leaders have access to their absolute positions in the global frame, whereas followers have access to relative positions of their neighbors within the sensing range. We \emph{assume} that there exists a time-invariant spanning forest\footnote{A spanning forest is a collection of disjoint spanning trees. 
$^2$A directed tree is a connected directed acyclic graph (DAG) where exactly one node (the root) has no parents, and every other node has exactly one parent.
$^3$Disjoint trees share no common vertex or node.}  $\mathcal{G}^{F} := (\mathcal{V}, \mathcal{E}^{F}, \mathcal{A}^{F})$ that is preserved under the time-varying edges of $\mathcal{G}(t)$. 
The time-invariant spanning forest $\mathcal{G}^{F}$, required as a prerequisite, is structured as a forest of directed trees$^2$, one per leader, growing outward from the leaders through successive generations of followers. 

\begin{definition}[$q$-Rooted Directed Forest]
The graph $\mathcal{G}^{F}$ is a $q$-rooted directed forest if it decomposes into $q$ disjoint directed trees$^3$, rooted at a distinct set of $q$ nodes, i.e., $\lb1, q \rb_{\mathbb N}$. 
\end{definition}
In our case, we consider $\mathcal{G}^{F}$ as an $n_l$-Rooted Directed Forest, rooted at distinct leaders. Each tree 
$\mathcal{T}_l:=
(\mathcal{V}_{\mathcal{T}_l},
\mathcal{E}_{\mathcal{T}_l},
\mathcal{A}_{\mathcal{T}_l})$ is 
rooted at leader $l\in\mathcal{V}_L$. The vertex set of each tree $\mathcal{T}_l$ is denoted by $\mathcal{V}_{\mathcal{T}_l}$, and the collection $\{\mathcal{V}_{\mathcal{T}_l}\}_{l \in \mathcal{V}_L}$ forms a partition of $\mathcal{V}$. 
Next, we organize agents across all trees into distinct layers, and enforce a top-down structure on each tree $\mathcal{T}_l$ with a uniform branching factor $k \geq 1$ and an interleaved Breadth-First Indexing (iBFI) across layers, inspired by BFI in \cite{chepoi1998distance}.
\begin{figure}[!h]
\centering
\includegraphics[scale = 0.6]{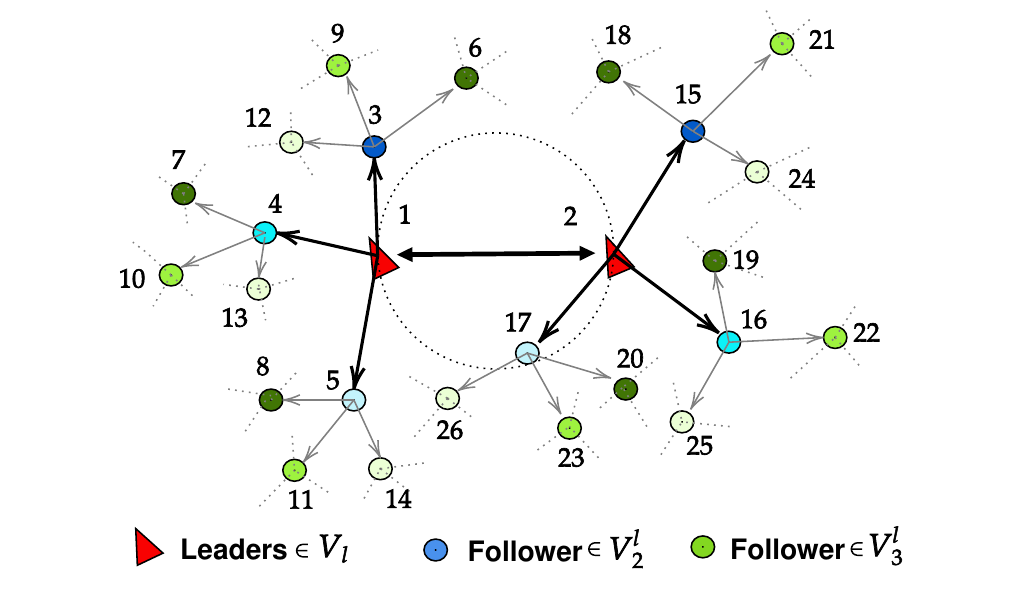}
\caption{Illustration of iBFI in iLDAG with $n_l = 2$, $k = 3$, $d = 3$. Leaders (in red) form a subset $V_1$. 
Subset $V_2$ agents are shown in blue and subset $V_3$ in green, with color intensity increasing with sibling rank $\sigma_{pr(i)}(i)$: the darkest node corresponds to $\sigma = 1$, and the lightest to $\sigma = 3$.}
\label{CBRS: fig/com_graph}
\end{figure}
\begin{definition}[Interleaved LDAG]
\label{CBRS: def/iLDAG}
Given an $n_l$-rooted directed forest $\mathcal{G}^{F}:= (\mathcal{V}, \mathcal{E}^{F}, \mathcal{A}^{F})$, we define the Interleaved Layered Directed Acyclic Graph (iLDAG), denoted by $\mathcal{G}^{F}_{\mathrm{iLDAG}}$, as the structure obtained when each tree $\mathcal{T}_l$, with vertex set $\mathcal{V}_{\mathcal{T}_l}$, is partitioned into $d$ disjoint layers: $\mathcal{V}_{\mathcal{T}_l} = \bigcup_{s=1}^{d} V_s^{(l)}$ with $V_1^{(l)} = \{l\}, \bigcup_{s=2}^{d} V_s^{(l)} = \mathcal{V}_{\mathcal{T}_l} \setminus \{l\}, \forall l \in \mathcal{V}_L$, and there exists a linear extension ordering on $\mathcal{V}$ satisfying the following rules.
\begin{enumerate}
    \item \textit{Interleaved Breadth-First Indexing:} Agents are indexed layer by layer in increasing order for all $l \in \mathcal{V}_L$, where $V_1^{(l)} = \{l\},$ receives the lowest indices: $\bigcup_{l \in \mathcal{V}_L}V_1^{(l)} = \lb 1, n_l \rb_{\mathbb{N}}$, followed by $V_2^{(l)}$, and so on, such that every edge $(i,j) \in \mathcal{E}^{F}$ with $i \in V_s^{(l)}$, satisfies $j \in V_{s-1}^{(l)}$ under the order $i > j$. Also, within each layer $V_s^{(l)}$, let $\sigma_j(i) \in \lb1, k\rb_{\mathbb{N}}$, where $k$ is the branching factor, denote the sibling rank (birth-order) of agent $i$ among the children of its parent $j$. For leaders with no parents, $\sigma_0(l) = k, \forall l \in \mathcal{V}_{L}$.  \\
    More importantly, agents $i \in V_{s}^{(l)}$ are distributed such that all parents $j \in V_{s-1}^{(l)}$ in the preceding layer receive their {$r^{th}$ ranked child (where $r = \sigma_j(i)$ for child $i$) before any parent receives their $(r+1)^{th}$ ranked child.} 
    All agents in $V_s^{(l)}, \forall l \in \mathcal{V}_L$ have the same sensing range. 
    \item \textit{Recursive Directed Star Expansion:} Each tree $\mathcal{T}_{l}$ is {recursively defined} at every node $j\in\mathcal{V}_{\mathcal{T}_l}$. Let the child set of node $j$ be $\mathcal{S}_j:= \{\,i\in\mathcal{V}_{\mathcal{T}_l}:(i,j)\in\mathcal{E}_{\mathcal{T}_l}\,\}.$ Each agent $j\in V_s^{(l)}$, for $s\in\lb 1,d-2\rb_{\mathbb{N}}$, has exactly $k$ children, i.e., $|\mathcal{S}_j|=k.$
    Then, the subtree rooted at node $j$ is recursively defined as $ \mathcal{V}_{\mathcal{T}_j}
 :=
 \{j\}\cup
 \bigcup_{i\in\mathcal{S}_j}
\mathcal{V}_{\mathcal{T}_i}, 
\mathcal{E}_{\mathcal{T}_j}
:=
\bigcup_{i\in\mathcal{S}_j}
\left(
\mathcal{E}_{\mathcal{T}_i}\cup\{(i,j)\}
\right), 
\mathcal{A}_{\mathcal{T}_j}
:=
\begin{bmatrix}
0 & \mathbf{0}^{\top} \\
\boldsymbol{c}_j &
\mathrm{diag}
\left(
\mathcal{A}_{\mathcal{T}_{i_1}},
\dots,
\mathcal{A}_{\mathcal{T}_{i_k}}
\right)
\end{bmatrix},$
where $\mathcal{S}_j=\{i_1,\dots,i_k\}$ and $\boldsymbol{c}_j
=\begin{bmatrix}
a_{i_1j}\mathbf{e}_{n_{i_1}} &
\cdots &
a_{i_kj}\mathbf{e}_{n_{i_k}}
\end{bmatrix}^{\top},$ with $n_{i_r}=|\mathcal{V}_{\mathcal{T}_{i_r}}|$. Agents $j\in V_d^{(l)}$ have no children, i.e., $\mathcal{S}_j=\emptyset$.

\item \textit{Leaders rigid subgraph:} All the leaders $l \in \mathcal{V}_{L}$ form a rigid subgraph (as defined in Laman's Theorem~{\cite[Theorem 5]{anderson2003operations}}) with at least $2n_l - 3$ edges, making the minimum edge count as $\underbrace{n - n_l}_{trees} \ + \underbrace{2n_l - 3}_{leader \ subgraph} = n + n_l - 3.$
\end{enumerate}
\end{definition}
\begin{remark}
The rigidity requirement on the leaders' subgraph in Definition \ref{CBRS: def/iLDAG} is not conservative, as leaders can use global information to coordinate and acquire rigidity.
\end{remark}
Note that in iLDAG, since each non-leaf agent $i$, i.e., $i \in V_s^{(l)}, s \in \lb 1, d-1\rb_{\mathbb N}$, can have at most $k$ children, the layer $s$ capacity satisfies $\sum_{l = 1}^{n_l}|V_s^{(l)}| \leq n_l  k^{s-1}$. Hence, the capacity of the iLDAG network is defined as \\$\mathcal{N}_k
:=
n_l \sum_{s=1}^{d} k^{s-1}
=
\begin{cases}
n_l \dfrac{k^{d}-1}{k-1}, & k > 1, \\
n_l d, & k = 1.
\end{cases}$

\begin{table}[t]
\centering
\caption{Summary of Notations for iLDAG Spanning Forest-- Definition \ref{CBRS: def/iLDAG}}
\label{CBRS: tab/notations}
\renewcommand{\arraystretch}{1.1}
\setlength{\tabcolsep}{6pt}

\begin{tabular}{|c|p{8cm}|}
\hline
\textbf{Notation} & \textbf{Description} \\
\hline

$\mathcal{T}_l$ & Tree rooted at leader $l$ \\
\hline
$\mathcal{V}_{\mathcal{T}_l}$ & Vertex set of tree $\mathcal{T}_l$ \\
\hline
$V_s^{(l)}$ & Layer-$s$ agents of $\mathcal{T}_l$ \\
\hline
$d$ & Number of layers \\
\hline
$k$ & Branching factor \\
\hline
$\mathcal{S}_j$ & Child set of agent $j$ \\
\hline
$pr(i)$ & Parent of agent $i$ \\
\hline
$\sigma_{pr(i)}(i)$ & Sibling rank of child $i$ \\
\hline
$(i,j)$ & Edge from parent $j$ to child $i$ \\
\hline
$\lambda_i$ & Sensing range of agent $i$ \\
\hline
$\boldsymbol{p}_i(t)$ & Position of agent $i$ \\
\hline
$\mathcal{N}_k$ & Maximum network capacity \\
\hline
$\mathcal{G}(t)$ & Communication graph \\
\hline
$\hat{\mathcal{G}}(t)$ & Maintenance graph \\
\hline
$\mathcal{E}(t)$ & Directed edge set \\
\hline
$\hat{\mathcal{E}}(t)$ & Maintenance links \\
\hline
$\mathcal{G}^{F}$ & Spanning forest \\
\hline
${\mathcal{G}}^{C}(t)$ & Time-varying subgraph \\
\hline
$\mathcal{E}^{C}(t)$ & Time-varying edges \\
\hline
$\mathcal{E}^{F}$ & Spanning tree edges \\
\hline

\end{tabular}
\end{table}










To reduce the notational burden on the reader, a summary of the iLDAG notations is provided in Table \ref{CBRS: tab/notations}. The iLDAG construction is further illustrated through the following example.
\begin{example}
\label{ex:interleaved_iLDAG}
Consider $n = 26$ agents with $n_l = 2$ leaders and branching factor $k = 3$ 
(cf.\ Fig.~\ref{CBRS: fig/com_graph}). The leader set (in red in Fig.~\ref{CBRS: fig/com_graph}) is $\mathcal{V}_L = \{1, 2\}$ with $V_1^{(1)} = \{1\}$ and $V_1^{(2)} = \{2\}$. The corresponding tree vertex sets are $\mathcal{V}_{\mathcal{T}_1} = \{1\} \cup \lb 3, 14 \rb_{\mathbb{N}}$, $\mathcal{V}_{\mathcal{T}_2} = \{2\} \cup \lb 15, 26 \rb_{\mathbb{N}}$.

Tree $\mathcal{T}_{1}$ is filled first. Leader $1$ has $k = 3$ children, 
giving layer $V_2^{(1)} = \{3, 4, 5\}$ (circles with blue shades in Fig.~\ref{CBRS: fig/com_graph}). For layer $V_3^{(1)}$ (circles with green shades in Fig.~\ref{CBRS: fig/com_graph}), the 
iBFI assigns children of lower sibling rank (shown in Fig.~\ref{CBRS: fig/com_graph} with high color intensity in each layer) across all parents 
in $V_2^{(1)}$ before proceeding to the next sibling rank (shown in Fig.~\ref{CBRS: fig/com_graph} with decreasing color intensity in each layer): $\underbrace{3 \to 6,\; 4 \to 7,\; 5 \to 8}_{\sigma_{pr(i)}(i)=1},\quad
    \underbrace{3 \to 9,\; 4 \to 10,\; 5 \to 11}_{\sigma_{pr(i)}(i)=2},\\[6pt]
    \underbrace{3 \to 12,\; 4 \to 13,\; 5 \to 14}_{\sigma_{pr(i)}(i)=3}.$
Since $V_3^{(1)}$ agents have no children left, tree $\mathcal{T}_{1}$ terminates with 
$d = 3$ layers and $|\mathcal{V}_{\mathcal{T}_{1}}| = 13$.

Tree $\mathcal{T}_{2}$ follows identically starting from leader $2$, yielding $V_2^{(2)} = \{15, 16, 17\}$ and $V_3^{(2)} = \{18, \ldots, 26\}$. Every edge $(i, pr(i))$ satisfies $i > pr(i)$, confirming the linear extension ordering. The leaders forms a ring (as shown with a bidirectional arrow in Fig. \ref{CBRS: fig/com_graph}). The total edge count is $|\mathcal{E}^{F}| = (26 - 2) + 1 = 25 = n + n_l - 3$, and the network capacity $\mathcal{N}_3 = 2 \cdot \frac{3^3 - 1}{2} = 26 = n$ is exactly met. \uqed
\end{example} 
Under this iLDAG graph structure, for any $(i,j) \in \mathcal{E}^{F}$, the relative position vector is defined as
$\boldsymbol{p}_{ij} := \boldsymbol{p}_j - \boldsymbol{p}_i$. The family of agent $i$ is denoted by $\{{pr}(i)\} \cup \mathcal{S}_{{pr}(i)}$.

\subsection{Rigidity of the Maintenance Graph}

For maneuvering multi-agent systems in formation with a minimum number of communication links, network rigidity is necessary (referred to as minimal rigidity in \cite{anderson2018rigid}). Network rigidity ensures that the entire formation in MAS can be maintained as a rigid body with few maintenance links (edges) $\hat{\mathcal{E}} \subset \mathcal{E}^{*}$, rather than the entire edge set $\mathcal{E}^{*}:= \{(i,j) \in \mathcal{V} \times \mathcal{V} : i \neq j\}$ of the complete graph. Under the iLDAG network $\mathcal{G}^{F}_{\mathrm{iLDAG}}$ in Definition~\ref{CBRS: def/iLDAG}, each follower has exactly one parent, forming a collection of directed trees with at least $n + n_l -3$ edges. A graph on $n$ vertices is not rigid considering the initial tree edges alone (cf. Example \ref{CBRS: ex/treenr}). For network rigidity in 2D, at least $2n-3$ maintenance links (or edges) are required, as established for 2D minimally rigid formations in \cite{anderson2003operations}. In addition, these $2n-3$ edges must satisfy certain structural conditions, discussed later in Theorem \ref{CBRS: thm/laman}. 

In this work, the sensing range $\lambda_i$ determines the creation and removal of edges. The sensing range $\lambda_i$ of each $s$-layer agent $i\in V_s^{(l)}$ induces edges to all agents within $\mathbb{B}(\boldsymbol{p}_i, \lambda_i)$, potentially including siblings within the same family ($\{{pr}(i)\} \cup \mathcal{S}_{{pr}(i)}$), agents from different layers, or agents across different trees. Managing all such edges does not scale well and fails to provide a systematic framework for acquiring rigidity. Thus, next we select a subset of sensed edges as \emph{maintenance links} $\hat{\mathcal{E}}$ and construct a subgraph, called the \emph{maintenance graph}, on which rigidity will be defined.
\begin{definition}[Maintenance Graph and Maintenance Links]
\label{CBRS: def/formation_graph}
Given the iLDAG communication subgraph $\mathcal{G}^{F}_{\mathrm{iLDAG}}$ in the entire graph $\mathcal{G}$, the maintenance link set $\hat{\mathcal{E}}\subseteq \mathcal{E}^{*}$ and the maintenance graph $\hat{\mathcal{G}}=(\mathcal{V},\hat{\mathcal{E}})$ are constructed based on a rank and topology-dependent rule. Let agent $i \in V_s^{(l)}$ in layer $s$ for the tree rooted at the leader $l$, with parent $pr(i) \in V_{s-1}^{(l)}$ (under the iLDAG rule (2)) and parent's sibling rank $\sigma_{pr^{(2)}(i)}(pr(i))$. Then, an edge $(i,j)$ is admissible as a maintenance link if any of the following holds.\\
(i) it satisfies the iLDAG construction rule $i > j, j=pr(i),$ or\\
(ii) if $\sigma_{pr^{(2)}(i)}(pr(i))<k$, then the $(\sigma_{pr^{(2)}(i)}(pr(i))+1)^{\text{th}}$ ranked sibling in $V_{s-1}^{(l)}$, within the same tree $\mathcal{T}_l$, acts as the parent $j$, or\\
(iii) if $\sigma_{pr^{(2)}(i)}(pr(i))=k$, then the first ranked sibling in $V_{s-1}^{(l')}$ of the lowest index, in the adjacent tree $\mathcal{T}_{l'}$, acts as the parent $j$, where $l' = (l \bmod n_l) + 1$ under cyclic ordering.
\end{definition}
We now formally define the rigidity for the maintenance links $\hat{\mathcal{E}}$ in the maintenance graph $\hat{\mathcal{G}}$.

\begin{definition}[Rigid Formation~{\cite{eren2004rigid}}]
\label{CBRS: def/rigid_form}
The maintenance graph $\hat{\mathcal{G}}$ with agent positions $\boldsymbol{p}_i\in\mathbb{R}^2$, $\forall i \in \mathcal{V}$, is \emph{rigid} in $\mathbb{R}^2$ if every smooth motion that preserves $\|\boldsymbol{p}_{ij}\|$ for all $(i,j) \in \hat{\mathcal{E}}$ preserves the distance between every pair of agents.
\end{definition}
The above definition of maintenance graph rigidity is further simplified in terms of rules on the set of maintenance links $\hat{\mathcal{E}}$:
\begin{theorem}[Laman's Theorem~{\cite[Theorem 5]{anderson2003operations}}]
\label{CBRS: thm/laman}
The directed maintenance graph $\hat{\mathcal{G}} = (\mathcal{V}, \hat{\mathcal{E}})$, evaluated as an undirected framework, is said to be rigid in $\mathbb{R}^2$ if and only if there exists a subset $\hat{\mathcal{E}}' \subseteq \hat{\mathcal{E}}$ with
\begin{enumerate}
    \item $|\hat{\mathcal{E}}'| = 2|\mathcal{V}| - 3$,
    \item $|\hat{\mathcal{E}}''| \leq 2|\mathcal{V}(\hat{\mathcal{E}}'')| - 3$ for all non-empty $\hat{\mathcal{E}}'' \subseteq \hat{\mathcal{E}}'$,
\end{enumerate}
where $\mathcal{V}(\hat{\mathcal{E}}'')$ is the set of vertices incident to edges in $\hat{\mathcal{E}}''$.
\end{theorem}

In the above theorem, the directed maintenance graph $\hat{\mathcal{G}}$ is evaluated as an undirected framework, since preserving the directed edge distance trivially preserves the undirected edge distance, needed for rigidity as in Definition \ref{CBRS: def/rigid_form}. 
However, achieving rigidity under limited sensing range is feasible only if inter-agent safety is guaranteed, which is discussed next.

\subsection{Inter-Agent Collision Avoidance}
 
During rigidity acquisition, agents move to establish cross-family maintenance links, augmenting the set of time-varying edges, often bringing agents into close proximity in terms of sensing range and creating a significant risk of collision. Moreover, agents attempting to acquire a required parent under sensing range constraint are obstructed by neighboring agents, where safety potentially leads to deadlocks in the rigidity acquisition process. Therefore, we next formalize the safety requirement as follows.

\begin{definition}[Inter-Agent Collision Avoidance]
\label{CBRS: def/collision_avoidance}
The multi-agent system satisfies inter-agent collision avoidance if
\begin{equation}
    \|\boldsymbol{p}_{ij}(t)\| \geq r_{ij}, \quad \forall\, i, j \in \mathcal{V},\; i \neq j,\; \forall\, t \in \mathbb{R}_{0}^+,
    \label{CBRS: eqn/safety}
\end{equation}
where $\boldsymbol{p}_{ij}(t) = \boldsymbol{p}_{j}(t) - \boldsymbol{p}_{i}(t), r_{ij} \in \mathbb{R}^+$ is the minimum safe distance between agents $i$ and $j$.
\end{definition} 

The above setup considers the pairwise safety interaction between agent $i$ and another agent $j$. In the distributed setting, every follower agent constructs several such constraints for all agents within its local vicinity, which is presented in Subsection \ref{CBRS: ssec/IVB}. Therefore, with the rigidity conditions and inter-agent safety defined, we now formulate the main problem.

\subsection{Problem Formulation}

Given the iLDAG network $\mathcal{G}^{F}_{\mathrm{iLDAG}}$ from Definition~\ref{CBRS: def/iLDAG}, each follower has a single tree parent, contributing $n + n_l - 3$ maintenance links to $\hat{\mathcal{E}}(0)$. By Theorem~\ref{CBRS: thm/laman}, rigidity in $\mathbb{R}^2$ requires $2n - 3$ links, leaving a deficit of $n - n_l$ edges. This operation requires each follower agent to enter the sensing proximity of the requisite second parent under limited sensing range constraint. Since this acquisition occurs while agents are in motion, inter-agent safety must be enforced throughout without deadlocks during acquisition. For achieving the objectives using only local information, we assume:
\begin{assum}
    The relative position $\boldsymbol{p}_{ij}(t)$, velocity $\boldsymbol{v}_{ij}(t)$, the control input norm $\|\boldsymbol{u}_j(t)\|$, Lipschitz constants $\Gamma_{jf}, \Gamma_{jh}$ of $j^{th}$ agent with respect to $i^{th}$ agent can be measured at time $t$ if they are connected, i.e., if $(i, j) \in \mathcal{E}^{F}$.
    \label{CBRS: assum/A1}
\end{assum}
\begin{remark}
Assumption \ref{CBRS: assum/A1} is non-conservative, requiring only relative neighbor measurements within the sensing range $\lambda$, which can be obtained via onboard sensors such as LiDAR or camera, while the control input norm is exchanged among neighboring agents through local communication. 
\end{remark} 
\begin{problem}
\label{CBRS: prob/main}
Given a planar MAS modeled as~\eqref{CBRS: eqn/sys}, connected with leader-follower graph $\mathcal{G}(t) : = (\mathcal{V}, \mathcal{E}(t), \mathcal{A}(t))$ containing the subgraph $\mathcal{G}^{F}_{\mathrm{iLDAG}}$ (as in Definition~\ref{CBRS: def/iLDAG}) at time $t$, satisfying Assumption \ref{CBRS: assum/A2} and \ref{CBRS: assum/A1}, with heterogeneous sensing ranges $\lambda_i \in \mathbb{R}^+, \forall i \in \mathcal{F} $, design a distributed control law such that the following conditions are satisfied.
\begin{enumerate}
    \item[(i)] \textit{Rigidity Acquisition:} There exists a finite $t' \in \mathbb{R}^+$ such that the maintenance graph $\hat{\mathcal{G}}(t')$, as in Definition~\ref{CBRS: def/formation_graph}, is rigid (i.e., $\hat{\mathcal{G}}(t')$ satisfies the conditions of Theorem~\ref{CBRS: thm/laman});
    \item[(ii)] \textit{Safety:} Inter-agent collision avoidance is ensured as in Definition~\ref{CBRS: def/collision_avoidance}.
\end{enumerate}
\end{problem}
To address Problem~\ref{CBRS: prob/main}, we develop a hierarchical splay strategy enforced on the time-invariant spanning forest $\mathcal{G}^{F}_{\mathrm{iLDAG}}$ for sensing range based edge acquisition in multi-robot systems.
\subsubsection*{Proposed solution}
 We design a CBF-QP like controller framework structured as shown in Fig. \ref{CBRS: fig/flow_graph}, given in Section \ref{CBRS: sec/IV}. We enforce geometric structure on $\mathcal{G}^{F}_{\mathrm{iLDAG}}$, explained in Section \ref{CBRS: sec/splay}, using the nominal controller defined in Theorem \ref{CBRS: thm/hetero_splay}, in Subsection \ref{CBRS: ssec/IVA}. At the final geometric structure (splay scheme) every follower agents acquires rigidity as defined in Theorem \ref{CBRS: thm/laman} using Lemma \ref{CBRS: thm/vertex_addition} if sensing range constraints are satisfied as in Lemma \ref{CBRS: lem/rigidity_st}. The final geometric shape is ensured safe with predefined inter-agent safety margin using Lemma \ref{CBRS: lem/safety}. But we assure inter-agent safety during rigidity acquisition by improved collision cone controller barrier functions (C3BF), capable of avoiding dynamic moving agents in following Subsection using Theorem \ref{CBRS: thm/c3bf}. Finally, the C3BF-QP framework is explained in Theorem \ref{thm:safe_splay_local_cbf}.
 Accordingly, the graph $\mathcal{G}(t)$ is decomposed into the fixed iLDAG network $\mathcal{G}^{F}_{\mathrm{iLDAG}}$ and a time-varying subgraph $\mathcal{G}^{C}(t) := (\mathcal{V}, \mathcal{E}^{C}(t), \mathcal{A}^{C}(t))$, such that $\mathcal{E}(t) = \mathcal{E}^{F} \cup \mathcal{E}^{C}(t)$ and $\mathcal{A}(t) = \mathcal{A}^{C}(t) + \mathcal{A}^{F}$.
The primary objective is to construct the required edges in the time-varying subgraph $\mathcal{G}^{C}(t)$ such that it guarantees the rigidity of the entire communication graph $\mathcal{G}(t)$. 

\begin{figure}[!h]
\centering
\includegraphics[scale = 1]{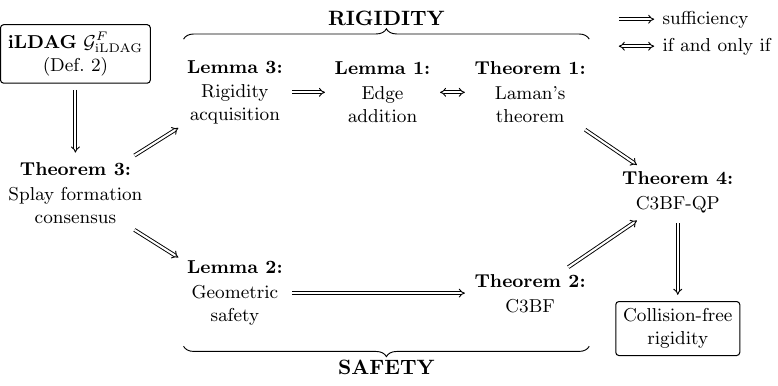}
\caption{Logical flow of the proposed solution.}
\label{CBRS: fig/flow_graph}
\end{figure}
The recursive directed star tree on $\mathcal{G}(t)$, $\forall t\in\mathbb{R}_{0}^{+}$, is imposed to ensure rigidity and safety for physical robots under limited sensing. In contrast, connectivity-only methods based on algebraic connectivity or the Fiedler eigenvalue \cite{ong2021network, yang2010decentralized} are suitable for rigidity in cases were safety is not required such as wireless sensors networks.

\subsection{Collision Cone Control Barrier Function}

To enforce the inter-agent safety condition \eqref{CBRS: eqn/safety} at all times, we employ the Collision Cone Control Barrier Function (C3BF) from~\cite{tayal2024collision}. C3BF is designed for obstacles moving with constant velocity but we extend it to accelerating obstacles by the following assumption.
\begin{assum}\label{CBRS: assum/A3}
    Given that the network $\mathcal{G}(t)$ has an iLDAG $\mathcal{G}^{F}_{\mathrm{iLDAG}}$ subgraph, with $pr(i), pr^{(2)}(i)$ as the parent and grandparent of each agent $i$ respectively, we assume that parent $pr(i)$ communicates to agent $i$, the recursively accumulated relative position $\mathcal{B}_p(i,pr(i)) = \|\boldsymbol{p}_{ipr(i)}\| + \mathcal{B}_p(pr(i), pr^{(2)}(i))$ and accumulated relative velocity $\mathcal{B}_v(i,pr(i)) = \|\boldsymbol{v}_{ipr(i)}\| + \mathcal{B}_v(pr(i), pr^{(2)}(i))$ to the agent $i$. Also, for parents that are leaders $\mathcal{B}_p(i,l) = \|\boldsymbol{p}_{l}\| + \|\boldsymbol{p}_{il}\|, \mathcal{B}_v(i,l) = \|\boldsymbol{v}_{l}\| + \|\boldsymbol{v}_{il}\|$, since they have global frame information.
\end{assum}

Moving forward, for any agent pair $(i,j)$ governed by the dynamics~\eqref{CBRS: eqn/sys}, agent $i$, we define agent $j$'s (treated as obstacle) relative position as $\boldsymbol{p}_{ij} := \boldsymbol{p}_j - \boldsymbol{p}_i$, relative velocity as $\boldsymbol{v}_{ij} := \boldsymbol{v}_j - \boldsymbol{v}_i$, and the half-angle of the collision cone as $\cos\phi_{ij} = \frac{\sqrt{\|\boldsymbol{p}_{ij}\|^2 - (r_{ij})^2}}{\|\boldsymbol{p}_{ij}\|}.$
The C3BF candidate for agent $i$, against agent $j$ as the obstacle, is defined as:
\begin{equation}
    b(\boldsymbol{p}_{ij}, \boldsymbol{v}_{ij}) := \langle \boldsymbol{p}_{ij},\, \boldsymbol{v}_{ij} \rangle - \|\boldsymbol{p}_{ij}\|\, \|\boldsymbol{v}_{ij}\|\, \cos(\pi -\phi_{ij}),
    \label{CBRS: eqn/c3bf}
\end{equation}
and the associated safe set is $\mathcal{C}_{ij} := \{(\boldsymbol{p}_{ij}, \boldsymbol{v}_{ij}) \in \mathbb{R}^2 \times \mathbb{R}^2 : b(\boldsymbol{p}_{ij}, \boldsymbol{v}_{ij}) \geq 0\}$. Similarly, the boundary set $\partial\mathcal{C}_{ij} := \{(\boldsymbol{p}_{ij}, \boldsymbol{v}_{ij}) \in \mathbb{R}^2 \times \mathbb{R}^2 : b(\boldsymbol{p}_{ij}, \boldsymbol{v}_{ij}) = 0\}$. Let us consider the system's state $\boldsymbol{x}_{ij}(t) = [\boldsymbol{p}_{ij}(t)^{\top},\ \boldsymbol{v}_{ij}(t)^{\top}]^{\top}$. 
\begin{theorem}
\label{CBRS: thm/c3bf}
Given a pair of agents $i, j \in \mathcal{V}$ governed by \eqref{CBRS: eqn/sys} satisfying Assumptions \ref{CBRS: assum/A2}- \ref{CBRS: assum/A3}, the safe set $\mathcal{C}_{ij}$, using the C3BF $b$ as in \eqref{CBRS: eqn/c3bf}, with $\frac{\partial b}{\partial \boldsymbol{x}_{ij}} \neq \boldsymbol{0}$ for all $\boldsymbol{x}_{ij} \in \partial\mathcal{C}_{ij}$, if $b(\boldsymbol{x}_{ij}(0)) \geq 0$, and the control input $\boldsymbol{u}_i = g_i(\delta_i)\boldsymbol{\xi}_{ij}\boldsymbol{\xi}_{ij}^{\top}\hat{\boldsymbol{u}}_i$, where $\boldsymbol{\xi}_{ij} = \boldsymbol{p}_{ij} + \cos{(\phi_{ij})\frac{\|\boldsymbol{p}_{ij}\|}{\|\boldsymbol{v}_{ij}\|}}\boldsymbol{v}_{ij}$ and  $\hat{\boldsymbol{u}}_i$ is computed as:
\begin{align}
    \hat{\boldsymbol{u}}_i &= \underset{\hat{\boldsymbol{q}}_i \in \mathbb{R}^2}{\arg\min} \; \|\hat{\boldsymbol{q}}_i\|^2, \nonumber \\
    \text{s.t.} &\quad L_{f_{ij}}' b(\boldsymbol{x}_{ij}) + L_{h_{ij}}' b(\boldsymbol{x}_{ij})\, \hat{\boldsymbol{q}}_i + \kappa(b(\boldsymbol{x}_{ij})) \geq 0, \nonumber \\
    & \quad \boldsymbol{\xi}_{ij}^{\top}\hat{\boldsymbol{q}}_i \leq 0,
    \label{CBRS: eqn/cbf_qp}
\end{align}
where $L_{f_{ij}}' b(\boldsymbol{x}_{ij}) = L_{f_{ij}} b   -\Gamma_{jf}\|\boldsymbol{x}_{ij}\|\|\boldsymbol{\xi}_{ij}\| - \bar{\Delta}\|\boldsymbol{\xi}_{ij}\| - \bar{\epsilon}\|\boldsymbol{u}_j\|\|\boldsymbol{\xi}_{ij}\| - \Gamma_{ih}\|\boldsymbol{u}_j\|\|\boldsymbol{x}_{ij}\|\|\boldsymbol{\xi}_{ij}\| - (\Gamma_{ih}+1)\|\boldsymbol{u}_j\|(\mathcal{B}_p(i,pr(i)) + \mathcal{B}_v(i,pr(i)))\|\boldsymbol{\xi}_{ij}\|$, $L_{h_{ij}}' b\, \hat{\boldsymbol{q}}_i = - (\Gamma_{ig}(\mathcal{B}_p(i,pr(i)) + \mathcal{B}_v(i,pr(i))))^{-1}\Gamma_i\|\boldsymbol{\xi}_{ij}\|^2\boldsymbol{\xi}_{ij}^{\top}\hat{\boldsymbol{q}}_i$, and $L_{f_{ij}} b$, $L_{h_{ij}} b$ are the Lie derivatives along \eqref{CBRS: eqn/sys}, represented as $\dot{\boldsymbol{x}}_{ij} = f_{ij}({\boldsymbol{x}_{ij}}) + h_{ij}({\boldsymbol{x}_{ij}}, \boldsymbol{u}_j)\boldsymbol{u}_i$, $\boldsymbol{x}_{ij}(t) = [\boldsymbol{p}_{ij}(t)^{\top},\ \boldsymbol{v}_{ij}(t)^{\top}]^{\top}, f_{ij} \in \mathbb{R}^{4}, h_{ij} \in \mathbb{R}^{4 \times 2}$, and $\kappa$ is an extended class-$\mathcal{K}_{e}$ function, then $\mathcal{C}_{ij}$ is forward invariant and the safety condition \eqref{CBRS: eqn/safety} holds for the pair $(i, j), \forall t \in \mathbb{R}_{0}^+$. 
\end{theorem}

\begin{proof}
For Proof refer to Appendix \ref{CBRS: apx/A}.
\end{proof}
\begin{remark}
Although restricting the search space to $\hat{\boldsymbol{u}}_i \in \hat{U}_i := \{\boldsymbol{u} \in \mathbb{R}^2 : \boldsymbol{\xi}_{ij}^{\top}\boldsymbol{u} \leq 0\}$ limits the admissible control directions, it is not inherently conservative. Since $\frac{\partial b}{\partial \boldsymbol{v}_i} = -\boldsymbol{\xi}_{ij}$, the constraint defines a state-dependent half-space aligned with the direction of steepest increase in safety, ensuring that the control input does not reduce the certified safety margin.
\end{remark}

\section{Splay Scheme in iLDAG Network}
\label{CBRS: sec/splay}
For solving Problem \ref{CBRS: prob/main} across an arbitrary number of agents under sensing range constraints, we propose an iLDAG geometry on the graph $\mathcal{G}^{F}_{\mathrm{iLDAG}}$ such that it acquires the requisite maintenance links while strictly satisfying the inter-agent safety.
\begin{figure}[!h]
\centering
\includegraphics[scale = 0.8]{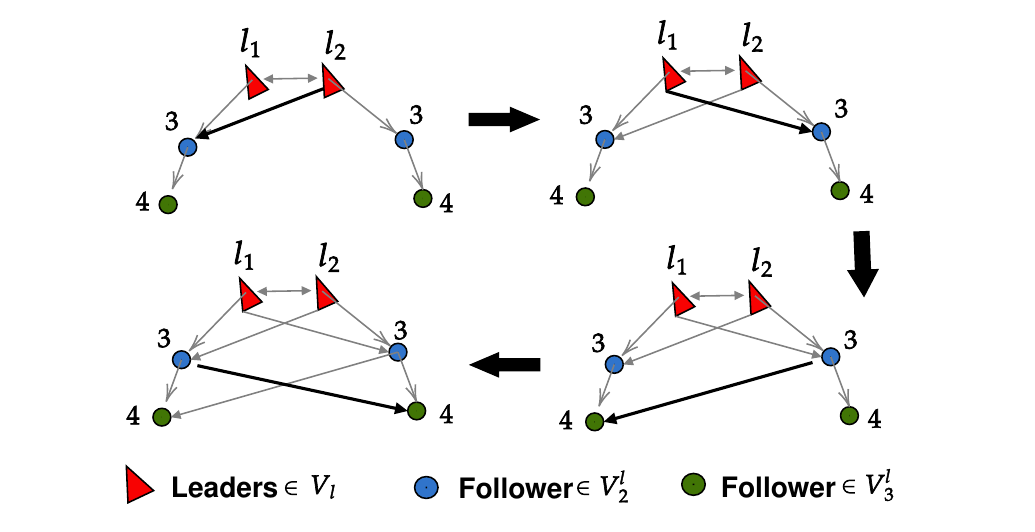}
\caption{Illustration of edge addition-based rigidity construction for the maintenance graph.}
\label{CBRS: fig/hen_graph}
\end{figure}
Although Theorem \ref{CBRS: thm/laman} provides a characterization of rigidity in terms of edge counts and subgraph constraints, for implementation using iLDAG geometry, a constructive method that builds a rigid maintenance graph step by step is more convenient. A standard approach is the Henneberg vertex addition method~\cite{anderson2003operations}, where starting from a single edge, the graph is grown by adding one vertex at a time, with each new vertex connects to exactly two existing vertices, preserving rigidity at every step. For our setting, with a fixed number of agents, we describe the procedure in terms of edge addition.
\begin{lemma}[Edge Addition-Based Rigidity Construction]
\label{CBRS: thm/vertex_addition}
Given a rigid maintenance graph $\hat{\mathcal{G}}_m = (\mathcal{V}_m, \hat{\mathcal{E}}_m)$ in $\mathbb{R}^2$ as discussed in Theorem \ref{CBRS: thm/laman}, consider a follower $i \in V_s^{(l)}$ as per the index ordering under the iLDAG structure. If agent $i$ admits exactly one tree parent $j_1 = {pr}(i) \in V_{s-1}^{(l)}$ as guaranteed by the iLDAG construction, and acquires one additional cross-family maintenance link $(i,j_2)$ with $j_2 \in V_{s-1}^{(l')}, j_2 \notin \mathcal{S}_{j_1}$, where $l' = l$ or $l' = (l \bmod n_l) + 1$, selected according to Definition~\ref{CBRS: def/formation_graph}, then the resulting graph $\hat{\mathcal{G}}_{m+1} =
(\mathcal{V}_m \cup \{i\},\ \hat{\mathcal{E}}_m \cup \{(i,j_1),(i,j_2)\})$ is rigid.
\end{lemma}
\begin{proof}
For Proof refer to Appendix \ref{CBRS: apx/B}.
\end{proof}
By the edge-addition sequence of Lemma~\ref{CBRS: thm/vertex_addition}, once the maintenance graph $\hat{\mathcal{G}}(t')$ becomes rigid at finite time $t'$, the communication graph $\mathcal{G}(t)$ is also rigid, since Theorem~\ref{CBRS: thm/laman} requires only a rigid subset of edges, provided by the maintenance graph. Accordingly, each follower augments its tree edge $(i,j_1)$ with a cross-family edge $(i,j_2)$ following Definition \ref{CBRS: def/formation_graph}, sequentially constructing the required deficit edges while preserving rigidity (cf. Fig.~\ref{CBRS: fig/hen_graph}).

During this process, agents may become trapped while simultaneously maintaining safety and acquiring rigidity due to the persistent presence of neighboring agents. To prevent such deadlocks in the rigid graph, we introduce a \emph{splay scheme} over the recursive directed-star spanning forest $\mathcal{G}^{F}_{\mathrm{iLDAG}}$, assigning each agent a unique target so it does not remain unnecessarily in another agent's vicinity, thereby ensuring safe and uninterrupted rigidity acquisition.
\begin{definition}[Splay scheme over iLDAG]
\label{CBRS: def/splay_form} 
Consider the iLDAG graph $\mathcal{G}^{F}_{\mathrm{iLDAG}}$ as defined in Definition~\ref{CBRS: def/iLDAG}. 
The agents are said to achieve a \textit{splay scheme} if for each tree $\mathcal{T}_l$ in the spanning forest $\mathcal{G}^{F}$ and for every agent $j \in V_s^{(l)}$, $s \in \lb 1, d-1 \rb_{\mathbb{N}}$, the children in $\mathcal{S}_j$ are positioned uniformly on a circle of radius $\eta_s \in \mathbb{R}^+$ centered at $j$. Also, all the leaders $l \in \mathcal{V}_{L}$ are positioned uniformly on a circle of radius $\eta_0$, centered at the origin of the global frame.
\end{definition}
Given the sibling rank $\sigma_j(i)\in\lb1,k\rb_{\mathbb{N}}$ of $(s+1)$-layer agent $i\in\mathcal{S}_j$, its splay equilibrium position relative to parent $j$ is
\begin{equation}
\boldsymbol{p}_i(t)=\boldsymbol{p}_j(t)+\eta_s
\begin{bmatrix}
\cos\!\left(\frac{2\pi(\sigma_j(i)-1)}{k}+\theta_j(t)\right)\\
\sin\!\left(\frac{2\pi(\sigma_j(i)-1)}{k}+\theta_j(t)\right)
\end{bmatrix},
\label{CBRS: eqn/splay}
\end{equation}
where $k$ is the branching factor in $\mathcal{G}^{F}_{\mathrm{iLDAG}}$, and $\theta_j(t)$ denotes the orientation of agent $j$'s local frame with respect to the global frame. This construction is recursively applied across all tree layers, with each agent serving as the splay center for its own children.

To guarantee cross-family link acquisition within the sensing range and construct the rigid maintenance graph $\hat{\mathcal{G}}(t)$ (Definition~\ref{CBRS: def/formation_graph}), the orbital geometry is bounded. Moreover, to ensure inter-agent safety with margin $r^a:=\max_{(i,j)}r_{ij}$ while maximizing agent accommodation, the orbit radius $\eta_s$ is reduced progressively across layers, as formulated below.
\begin{definition}[Geometric Orbit Decay]
The orbit radius $\eta_s$ for the agents in layer $V_s^{(l)}, \forall l \in \mathcal{V}_{L},$ is said to satisfy \emph{Geometric Orbit Decay} if $\eta_{s} = \eta_0 \alpha^{s}, \quad \forall s \in \lb 1, d-1 \rb_{\mathbb N},$
where $\alpha \in (0,1)$ is the radial decay rate, and $\alpha^s$ denotes the $s$-th power of $\alpha$.
\end{definition}
Since leader agents have no parents, we define their parent as the virtual zero node $pr(l)=0$ and denote the set of all leaders by $\mathcal{S}_0 := \mathcal{V}_L$ for notational consistency with follower families. Using the geometric decay structure on the orbit radius $\eta_s$, the safety constraints are formulated next for the iLDAG network.
\begin{lemma}[Geometric Safety Constraints]
\label{CBRS: lem/safety}
Given the splay scheme as in Definition \ref{CBRS: def/splay_form}, if, for any agent $i \in V_s^{(l)}$, the orbit radius $\eta_{s-1}$  of the parent $pr(i) (\in V_{s-1}^{(l)})$ and the decay rate $\alpha \in (0, \sin(\frac{\pi}{k}))$, obey the following bounds $\forall s \in \lb 1, d \rb_{\mathbb N}, \forall l \in \mathcal{V}_{L}$:
\begin{align}
    &\eta_{s-1} > \max \Big\{r^a, \frac{r^a}{2\sin(\pi/k)}, \frac{r^a}{2\big(\frac{1}{\alpha}\sin\left(\frac{\pi}{k}\right) - 1\big)}\Big\}, \nonumber \\
    &\eta_0 >\frac{r^a}{2\big( \sin(\frac{\pi}{n_l}) - \alpha_c\big)}, \alpha_c = \frac{\alpha - \alpha^d}{1 - \alpha},
    \label{eq:safety_bounds}
\end{align}
then the inter-agent safety as in Definition \ref{CBRS: def/collision_avoidance}, for the iLDAG spanning forest $\mathcal{G}^{F}_{\mathrm{iLDAG}}$, is satisfied with the safety margin $r^a = \max_{(i,j)} r_{ij}$, for any $i, j \in \mathcal{V}$.
\end{lemma}
\vspace{2mm}
\begin{proof}
For Proof refer to Appendix \ref{CBRS: apx/C}.
\end{proof}
To illustrate the geometric conditions for inter-agent safety in the splay scheme, we next present the procedure for deriving $\alpha$ satisfying safety with given $\eta_0$ through the following example.
\begin{example}
    Consider the iLDAG subgraph $\mathcal{G}^{F}_{\mathrm{iLDAG}}$ with the parameters as $\eta_0 = 25, r^a = 2, k=6, n_l = 4, d=3$. For $k=6$, $\max\{r^a, \frac{r^a}{2\sin(\pi/k)}\} = 2$. Let $s=3$ and derive the bounds for the first inequality in \eqref{eq:safety_bounds}: $\eta_2 = 25\alpha^2 > 2 \Rightarrow \alpha > \sqrt{\frac{2}{25}} = 0.28, 
        \eta_2 = 25\alpha^2 > \frac{2}{\frac{1}{\alpha} - 2} \Rightarrow 25\alpha - 50\alpha^2 - 2 > 0
        \Rightarrow 0.1 < \alpha <0.4, $
    consequently, for the second inequality in \eqref{eq:safety_bounds}: $\eta_0 = 25 > \frac{1}{0.707 - \alpha_c} \Rightarrow \alpha^2 + \alpha  -0.6671 < 0
        \Rightarrow 0 < \alpha < 0.4416.$
    Similarly, for $s=1,2$, the inequalities result in $\alpha > \max\{0.08, 0.28, 0.1\} = 0.28$ and $\alpha < \min\{0.48, 0.46, 0.4, 0.4416\} = 0.4$. So, the maximal decay rate $\alpha = 0.4$ ensures inter-agent safety in the splay scheme. \uqed
\end{example}
Apart from inter-agent safety, the splay scheme should achieve network rigidity as well. The required constraints on sensing range $\lambda$ are discussed next.
\begin{lemma}[Rigidity Acquisition under Splay]
\label{CBRS: lem/rigidity_st}
Given the splay scheme as in Definition \ref{CBRS: def/splay_form}, with decay rate $\alpha$ satisfying Lemma \ref{CBRS: lem/safety} and if every follower $i \in V_s^{(l)} (s>1)$ under the iLDAG tree parent $j_1 \in V_{s-1}^{(l)}$, acquires a cross-family maintenance link to a second parent $j_2 \in V_{s-1}^{(l)}$, satisfying the condition (ii) of maintenance graph Definition \ref{CBRS: def/formation_graph} (same tree cross-family link) under the sensing range $\lambda_{i} \geq\eta_{s-1}$ condition:
\begin{equation}
    \lambda_i \geq \eta_{s-2}\left(2\sin\left(\frac{\pi}{k}\right) + \alpha\right),
    \label{eq:rigidity_st}
\end{equation}
or to a second parent $j_2 \in V_{s-1}^{(l')}$ (where trees $(l,l')$ are adjacent), satisfying the condition (iii) of maintenance graph Definition \ref{CBRS: def/formation_graph} (adjacent tree cross-family link) under the sensing range $\lambda_{i}$ condition: $\lambda_i \geq 2\eta_0\sin(\frac{\pi}{n_l}) + 2\sum_{r=1}^{s-2}\eta_r + \eta_{s-1},$ then the maintenance graph is safe (as in Definition \ref{CBRS: def/collision_avoidance}) and rigid as discussed in Theorem \ref{CBRS: thm/laman}.
\end{lemma}
\begin{proof}
For Proof refer to Appendix \ref{CBRS: apx/D}.
\end{proof} 
Next we consider the constraints on the branching factor $k$ with the following example.
\begin{table}[h]
\centering
\caption{Safety and sensing constraints for different branching factors}
\label{CBRS: tab/k_constraints}
\begin{tabular}{|c|c|c|}
\hline
\textbf{Branching Factor $k$} & \textbf{Safety Constraint} & \textbf{Sensing Constraint} \\
\hline
5 & $r^a < 3.75$ & $\lambda_i > 15.76$ \\
\hline
6 & $r^a < 2$ & $\lambda_i > 14$ \\
\hline
7 & $r^a < 0.67$ & $\lambda_i > 12.68$ \\
\hline
\end{tabular}
\end{table}
\begin{example}
    Consider the iLDAG spanning forest $\mathcal{G}^{F}_{\mathrm{iLDAG}}$ with the parameters as $\eta_0 = 25, n_l = 4, d=3, \alpha = 0.4$. We need to determine the required branching factor $k$ which maximizes feasible safety margin $r^a$ and minimizes smallest feasible sensing range $\lambda_i$ of agent $i$. For $s=3$, we have bounds for \eqref{eq:safety_bounds}:
    $\eta_2 = 4 > \frac{r^a}{2\sin (\frac{\pi}{k})} \Rightarrow \sin(\frac{\pi}{k}) > \frac{r^a}{8}, r^a < 4, \sin(\frac{\pi}{k}) > \alpha\big(\frac{r^a}{8} + 1\big) = \frac{r^a}{20} + 0.4,$
    and for \eqref{eq:rigidity_st}, $\sin(\frac{\pi}{k}) \leq \frac{\lambda_i - 4}{20}$. We do not consider cross-tree sensing and safety constraints, as they are unaffected by $k$. Now, we can observe the trade-off between safety and sensing requirement for branching factor $k$ in Table \ref{CBRS: tab/k_constraints}. Here $k = 6$ for $s=3$, would be the an appropriate choice. Similarly, together with the bounds for $s=1,s=2$, we can conclude the required branching factor. \uqed
\end{example}
\section{Controller Design for Collision-free Rigidity Acquisition}
\label{CBRS: sec/IV}
We need to achieve the second-order formation consensus (on both position and velocity) for tracking the splay scheme.
\begin{definition}[Formation Consensus over iLDAG]
\label{CBRS: def/star_consensus}
Consider the splay scheme defined in Definition~\ref{CBRS: def/splay_form} for $\mathcal{G}^{F}_{\mathrm{iLDAG}}$. 
The agents are said to achieve \textit{second-order formation consensus at every local star $\mathcal{S}_j$} if, for each agent $j \in V_s^{(l)}$, $s \in \lb 1, d-1 \rb_{\mathbb{N}}$, and for all children $i \in \mathcal{S}_j$, the following holds, $\lim_{t \to \infty} \left\| \boldsymbol{p}_i(t) - \boldsymbol{p}_j(t) - \boldsymbol{p}_{ij}^*(t) \right\| = 0, 
\lim_{t \to \infty} \left\| {\boldsymbol{v}}_i(t) - {\boldsymbol{v}}_j(t) \right\| = 0,$
where $\boldsymbol{p}_{ij}^*(t)$ denotes the desired relative position defined by the splay scheme, i.e.,
\begin{equation}\label{eqn: desiredpt} 
\boldsymbol{p}_{ij}^*(t) = \eta_s \begin{bmatrix}
\cos\left( \frac{2\pi (\sigma_j(i)-1)}{k} + \theta_j(t) \right) \\
\sin\left( \frac{2\pi (\sigma_j(i)-1)}{k} + \theta_j(t) \right)
\end{bmatrix}.
\end{equation}
\end{definition}

Also, in the splay scheme, enforcing only child-to-parent orientation constraints with respect to any child $i( \in \mathcal{S}_j)$'s frame, as $\boldsymbol{p}_{ij} = \boldsymbol{p}_{ij}^*$, allows their common parent $j$ to act as a local alignment target. The orientation of the agents in star $\mathcal{S}_j$, for agent $j$'s frame, is still undetermined. This can lead to degenerate global configurations (like all agents $i \in \mathcal{S}_j$ align at $0^{\circ}$ for agent $j$'s local frame) where agents in star $\mathcal{S}_j$ do not align uniformly around the parent $j$ despite satisfying pairwise constraints. Introducing bidirectional (additional reverse) constraints ensures that $\boldsymbol{p}_{ij} = \boldsymbol{p}_{ij}^*, \boldsymbol{p}_{ji} = -\boldsymbol{p}_{ij}^*$, thereby enforcing internal uniformity within each star for parent's frame and eliminating asymmetric dependencies between parent and children. The additional reverse edges, opposite to linear extension ordering, supplement the iLDAG architecture, also breaking the acyclic structure. To gain further insight, consider a local star centered at agent $j \in \mathcal{V}$ with child set $\mathcal{S}_j = \{i_1, i_2, \dots, i_k\}$. We denote the spanning forest $\mathcal{G}^{F}_{\mathrm{iLDAG}}$ with reverse edges as $\mathcal{G}^{F}_r := (\mathcal{V}, \mathcal{E}^{F}_r, \mathcal{A}^{F}_r)$, where $\mathcal{E}^{F}_r := \bigcup_{j \in \mathcal{V}} \{(j,i): i \in \mathcal{S}_j\} \cup \mathcal{E}^{F}$ and $\mathcal{A}^{F}_r$ is the corresponding adjacency matrix. The Laplacian matrix of the graph $\mathcal{G}_{r}^{F}$, $\mathcal{L}_{r}^{F}:= \mathcal{D}_r^{F} - \mathcal{A}_r^{F}$, where $\mathcal{D}_r^{F}$ is the degree matrix of $\mathcal{G}_{r}^{F}$.\\
Let us define the induced local star subgraph for $\mathcal{G}^{F}_r$ with reverse edges as $\mathcal{G}_{r_j}^{F} := (\mathcal{V}_{r_j}^F, \mathcal{E}_{r_j}^{F}, \mathcal{A}_{r_j}^{F}),$
where $\mathcal{V}_{r_j}^{F} = \{j\} \cup \mathcal{S}_j \cup \{pr(j)\}, \mathcal{E}_{r_j}^{F} = \mathcal{E}^{F}_r \cap (\mathcal{V}_{r_j}^{F} \times \mathcal{V}_{r_j}^{F})$, and $\mathcal{A}^{F}_{r_j}$ is the corresponding adjacency matrix.  
The constant Laplacian $\mathcal{L}_{r_j}^{F}$ of the graph $\mathcal{G}_{r_j}^{F}$ has the following form.
\begin{align*}
\mathcal{L}_{r_j}^{F} =
\begin{bmatrix}
\sum\limits_{i \in \mathcal{S}_j \cup \{pr(j)\}} a_{ji} & -a_{j i_1} & \cdots & -a_{j i_k} & -a_{jpr(j)}\\
-a_{i_1 j} & a_{i_1 j} & \cdots & 0 & 0\\
\vdots & \vdots & \ddots & \vdots & \vdots\\
-a_{i_k j} & 0 & \cdots & a_{i_k j} & 0\\
-a_{pr(j)j}&0&\cdots&0&a_{pr(j)j}
\end{bmatrix}.
\end{align*}


Since, in general, $a_{ji} \neq a_{ij}$ due to heterogeneous sensing and directed communication, the Laplacian $\mathcal{L}_{r_j}^{F}$ becomes non-symmetric. 
The Laplacian is no longer guaranteed to be positive semi-definite or diagonalizable with real eigenvalues (also highlighted in \cite[Remark 2]{wang2025scalable}). 


So, with complex Laplace eigenvalues for each local star, we next establish the controller to implement the splay scheme. 

\subsection{Hierarchical Formation Consensus for Splay Scheme}
\label{CBRS: ssec/IVA}

The construction proceeds hierarchically by recursively enforcing the local splay configuration at each star, as specified in Definition~\ref{CBRS: def/splay_form}. This recursive enforcement ensures that each parent–child group attains a consistent circular arrangement, which is propagated across successive layers of the iLDAG. 
Note that due to the presence of bidirectional edges between parent-child, i.e.,  $(i,j), (j,i)$, as per Assumption \ref{CBRS: assum/A3}, agent $i$ has access to both $\mathcal{B}_p(i,j), \mathcal{B}_v(i,j)$ and $\mathcal{B}_p(j,i), \mathcal{B}_v(j,i)$. 
Induction of the local splay property for globally consistent splay scheme over the spanning forest with reverse edges $\mathcal{G}^{F}_r$ is discussed next.
\begin{theorem}[Splay Formation Consensus]
\label{CBRS: thm/hetero_splay}

Consider the heterogeneous multi-agent system, modeled as \eqref{CBRS: eqn/sys}, satisfying Assumption \ref{CBRS: assum/A2} - \ref{CBRS: assum/A3}, connected under iLDAG subgraph $\mathcal{G}^{F}_r$ in $\mathcal{G}(t)$, with edge weight as the ratio of sensing ranges, i.e., $a_{ij} = \tau \frac{\lambda_i}{\lambda_j}, \forall i,j \in \mathcal{V}, \tau \in \mathbb{R}^{+}$. Let each follower agent $i$ implement the control law defined as
\begin{align}
\label{eq: Thm4_nom}
\boldsymbol{u}_i =  
& -k_{p} \sum_{j \in \mathcal{S}_i \cup \{pr(i)\} } a_{ij}
\left(\boldsymbol{p}_i - \boldsymbol{p}_j - \boldsymbol{p}_{ij}^*\right) \nonumber \\
& -k_{v} \sum_{j \in \mathcal{S}_i \cup \{pr(i)\}} a_{ij}
\left(\boldsymbol{v}_i - \boldsymbol{v}_j\right) ,
\end{align}
with $\boldsymbol{p}_{ij}^*$ given by the splay scheme in Definition~\ref{CBRS: def/star_consensus}.
If the gain ratio satisfies
\begin{equation}
\frac{k_{v}^2}{k_{p}} > 
\frac{\big(\operatorname{Im}(\mu_{z}(\mathcal{L}^{F}_r))\big)^2}
{\operatorname{Re}(\mu_{z}(\mathcal{L}^{F}_r))|\mu_{z}|^2}, \forall z \in [1, n]_{\mathbb{N}}, \ \text{s.t.} \ \mu_z \neq 0,
\label{eq: gains}
\end{equation}
where $\mu_{z}(\mathcal{L}^{F}_r)$ are complex eigenvalues for $\mathcal{G}^{F}_r$'s graph Laplacian matrix $\mathcal{L}^{F}_r \in \mathbb{R}^{n \times n}$, and under bounded initial error $\|\boldsymbol{x}_{im}(0)\| = \|[\boldsymbol{p}_{im}^{\top}(0), \boldsymbol{v}_{im}^{\top}(0)]^{\top}\| < M, M \in \mathbb{R}^{+}$, $\forall t \in \mathbb{R}^+_0$,
\begin{align}
    \mu^* = \min_{\forall \mu_z, \mu_z \neq 0} \frac{k_v}{2}\operatorname{Re}(\mu_{z}(\mathcal{L}^{F}_r)) > C\gamma^* \geq C \gamma(t),\label{eq: compn}
\end{align}
where $\gamma(t) = \max_i(\max_{(i,j)} \gamma_{ij}(t)), {\gamma}_{ij}(t) = \Gamma_{if} + \sum_{m \in \mathcal{S}_i \cup \{pr(i)\}} a_{im}(\Gamma_{mh}\|\boldsymbol{x}_{im}(t)\|+ 2\Gamma_{mh}(\mathcal{B}_p(m,i) + \mathcal{B}_v(m,i)) + \bar{\epsilon})(k_p + k_v)$, $C = \frac{\lambda_{\max}}{\lambda_{\min}}, \lambda_{\max} = \max_z \lambda_z, \lambda_{\min} = \min_z \lambda_z$, $\Gamma_{if}, \Gamma_{ih}, \Gamma_{jh}$ are Lipschitz constants for functions $f_{v_i}, h_{v_i}, h_{v_j}$, respectively. Then each local star achieves second-order formation consensus with an ultimate bound on tracking error as $
\|\boldsymbol{y}\| = \sum_{j \in \mathcal{S}_i \cup \{pr(i)\}} a_{ij} (\|\boldsymbol{p}_i - \boldsymbol{p}_j - \boldsymbol{p}_{ij}^*\| + 
\|\boldsymbol{v}_i - \boldsymbol{v}_j\|) \leq \big(\frac{C\Delta}{\mu^{*} - C\gamma^*}\big) + \varepsilon,$
where $\varepsilon \in \mathbb{R}^+, \Delta = \max_i (\max_{(i,j)} \Gamma_{if} \|\boldsymbol{p}_{ij}^*\| + \bar{\Delta})$, for all $t \geq T(\varepsilon)$, where $T(\varepsilon) := \frac{1}{\mu^{*} - C\gamma^*}  \ln\!\left( \frac{C\|\boldsymbol{y}(0)\| - \frac{C\Delta}{\mu^{*} - C\gamma^*}}{\varepsilon} \right)$. 
\end{theorem}
\begin{proof}
For Proof refer to Appendix \ref{CBRS: apx/E}
\end{proof}
\begin{remark}
    With $\mu^{*} \gg C\gamma^*$ using parameter $k_v$, we can make the worst-case tracking error arbitrarily small. Also, for system dynamics such as double-integrator, differential-drive, with $f_{v_i} = \boldsymbol{0}, \Delta = 0$, the worst-case tracking error, $\frac{C\Delta}{\mu^{*} - C\gamma^*} + \varepsilon = \varepsilon$.
\end{remark}

\begin{remark}[On bounding $\gamma(t)$]
The requisite condition \eqref{eq: compn} is not conservative: with the initial error bounded, $\|\boldsymbol{y}(0)\|\le R_0$ and $\gamma(0)\le R_0'$, $R_0, R_0'\in\mathbb{R}^+$, the invariance condition $\mu^{*}>C\gamma^{*}$ makes the set $\mathbb{B}(\boldsymbol{0}, R)$, with $R=CR_0+\tfrac{C\Delta}{\mu^{*}-C\gamma^{*}}$, forward invariant, so that $\|\boldsymbol{y}(t)\|\le R$ for all $t\in\mathbb{R}_0^{+}$. The admissible $R_0$, that is, the region of attraction, is fixed by $\mu^{*}>C\gamma^{*}$ with $\gamma^{*}=R_0'+\gamma_1 R$, where $\gamma_1\propto\Gamma_{ih}$, $\gamma_1\in\mathbb{R}^+$, and $\Gamma_{ih}$ is the Lipschitz constant of the input-gain function $h_{v_i}$. 
Hence a larger sensing-range spread $C=\lambda_{\max}/\lambda_{\min}$ or a larger $\Gamma_{ih}$ raises $C\gamma^{*}$ and shrinks $R_0$ for a given threshold $R$.
\end{remark}


\subsection{C3BF-QP Formulation for Collision-free Splay Scheme}
\label{CBRS: ssec/IVB}

For each agent $i \in \mathcal{V}$, the collision avoidance constraints are constructed only with respect to agents in its local vicinity, denoted as the set $\tilde{\mathcal{N}}_i(t) := \{m \in \mathcal{V} : \|\boldsymbol{p}_{im}\| \leq \tilde{\lambda}_i < \lambda_i\}$. Local vicinity radius $\tilde{\lambda}_i \in \mathbb{R}^{+}$ is chosen such that at splay equilibrium, the set $\tilde{\mathcal{N}}_i(t)$ contains only the required two parents.
Accordingly, the control barrier functions $b(\boldsymbol{x}_{ij})$ (as in \eqref{CBRS: eqn/c3bf}) are defined solely for agents $j \in \tilde{\mathcal{N}}_i(t)$, such that the resulting collision cones are locally induced and do not require global pairwise evaluation. This ensures scalability of the safety filter with respect to the iLDAG topology. The controller framework for collision-free rigidity acquisition is as follows.
\begin{theorem}[Safe splay scheme - C3BF-QP Control]
\label{thm:safe_splay_local_cbf}

Consider the heterogeneous multi-agent system \eqref{CBRS: eqn/sys} with limited sensing range $\lambda_i, \forall i \in \mathcal{V}$, evolving over an iLDAG with reverse edges $\mathcal{G}^{F}_r$ in $\mathcal{G}(t)$, satisfying Lemma \ref{CBRS: lem/safety} and \ref{CBRS: lem/rigidity_st}, and Assumption \ref{CBRS: assum/A2} - \ref{CBRS: assum/A3}. With $\boldsymbol{u}_i^{nom}$ as the nominal controller defined in \eqref{eq: Thm4_nom}, satisfying conditions of Theorem \ref{CBRS: thm/hetero_splay} and the control input $\boldsymbol{u}_i = g_i(\delta_i)\boldsymbol{\xi}_{ij}\boldsymbol{\xi}_{ij}^{\top}\hat{\boldsymbol{u}}_i$, where $\boldsymbol{\xi}_{ij} = \boldsymbol{p}_{ij} + \cos{\phi_{ij}\frac{\|\boldsymbol{p}_{ij}\|}{\|\boldsymbol{v}_{ij}\|}}\boldsymbol{v}_{ij}$, defined on $ \boldsymbol{x}_{ij} = [\boldsymbol{p}_{ij}^{\top}, \boldsymbol{v}_{ij}^{\top}]^{\top}$ following Theorem \ref{CBRS: thm/c3bf} is computed:
\begin{align}
\hat{\boldsymbol{u}}_i &= \arg\min_{\hat{\boldsymbol{q}}_i\in \mathbb{R}^2}  \|(g_i(\delta_i)\boldsymbol{\xi}_{ij}\boldsymbol{\xi}_{ij}^{\top})\hat{\boldsymbol{q}}_i - \boldsymbol{u}_i^{nom}\|^2, \nonumber \\
\text{s.t.}\ L_{f_{ij}}' &b(\boldsymbol{x}_{ij}) + L_{h_{ij}}' b(\boldsymbol{x}_{ij})\hat{\boldsymbol{q}}_i + \kappa (b(\boldsymbol{x}_{ij})) \geq 0, \forall j \in \tilde{\mathcal{N}}_i(t), \nonumber \\
& \boldsymbol{\xi}_{ij}^{\top}\hat{\boldsymbol{q}}_i \leq 0, \forall j \in \tilde{\mathcal{N}}_i(t),
\end{align}
under the condition that $b(\boldsymbol{x}_{ij}(0)) \geq 0$ for all $j \in \tilde{\mathcal{N}}_i(t)$, $\frac{\partial b}{\partial \boldsymbol{x}_{ij}} \neq \boldsymbol{0}$ on $\partial \mathcal{C}_{ij}$. Then, all pairwise safety sets $\mathcal{C}_{ij}$ are forward invariant, and the closed-loop system achieves safe second-order formation consensus for all agents $i \in \mathcal{F}$ when $t \geq T(\varepsilon)$, where $T(\varepsilon)$ is as defined in Theorem \ref{CBRS: thm/hetero_splay}. This solves Problem \ref{CBRS: prob/main}.
\end{theorem}
\begin{proof}
Forward invariance of the collision cones $\mathcal{C}_{ij}$, and hence inter-agent safety, follows directly from Theorem~\ref{CBRS: thm/c3bf}.

By Lemma~\ref{CBRS: lem/safety}, every non-parent neighbor $j \in \tilde{\mathcal{N}}_i(t)$ converges to its splay orbit, so $\|\boldsymbol{p}_{ij}\|>\tilde{\lambda}_i>r_{ij}$ eventually holds, making the corresponding C3BF constraints inactive. The nominal controller of Theorem \ref{CBRS: thm/hetero_splay} is thus recovered in finite time, enabling agent $i$ to establish the requisite cross-family link with its second parent $j_2$ (Lemma~3).

Since safety is enforced locally while consensus propagates recursively over the hierarchical iLDAG, spanning-tree connectivity is preserved. Consequently, the safe splay scheme extends network-wide, yielding the rigid maintenance graph $\hat{\mathcal{G}}(t)$ by Theorem~\ref{CBRS: thm/hetero_splay}, with rigidity guaranteed by Lemma~\ref{CBRS: lem/rigidity_st} acquiring the requisite $n - n_l$ maintenance links.
\end{proof}
\begin{figure}[!t]
\centering
\includegraphics[width = \columnwidth]{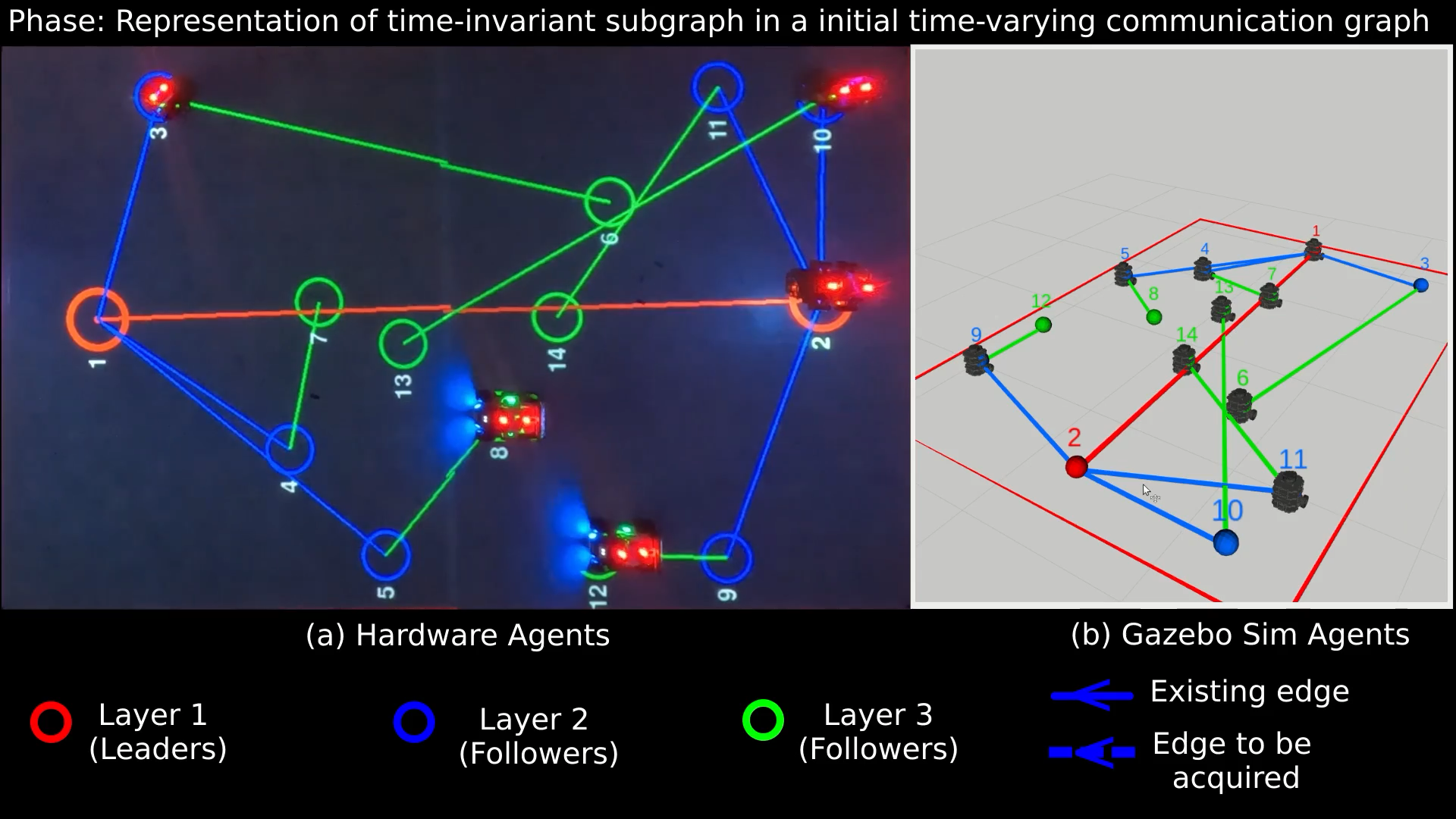}
\caption{Illustration of hardware experiment setup with 5 physical robots and 9 virtual gazebo robots. \\Video: \url{https://youtu.be/PyHoO0E9oLg}.}
\label{CBRS: fig/exp_stp}
\end{figure}

\begin{figure*}
    \centering
    \includegraphics[width=\linewidth]{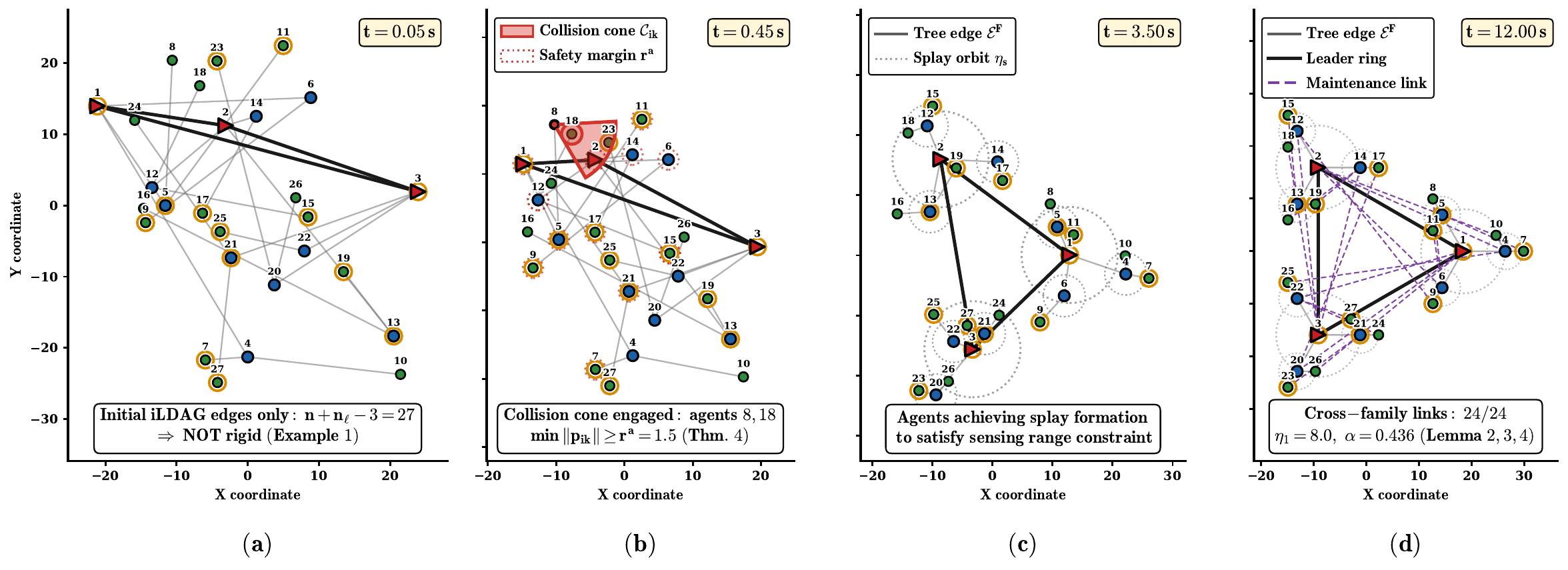}%
    \caption{ MATLAB Simulation results validating the proposed framework. Video: \url{https://youtu.be/hQlC0wjAk2E}. }
    \label{CBRS:simulation_results}
\end{figure*}

\section{Results \& Discussion}
\label{CBRS: sec/Resultswithc}

We validate the proposed framework through simulation with $n = 27$ agents in the 2D plane, excluding the virtual zero node, starting from random initial points with non-rigid network. Experimental hardware validation with plots for formation error and control inputs is available at: \\ \url{https://youtu.be/PyHoO0E9oLg}, and the setup is shown in Fig. \ref{CBRS: fig/exp_stp}. For simulation, the agents are modeled as double-integrators (shown as circles in Fig. \ref{CBRS:simulation_results}) with $f_{v_i} = \boldsymbol{0}$ and $h_{v_i} = I_2$ or as differential-drives (shown as circles with orange ring in Fig. \ref{CBRS:simulation_results}) with $f_{v_i} = 0$ and $h_{v_i} = { \left[ \begin{matrix} \cos(\delta_i) & -\beta_i \sin(\delta_i) \\ \sin(\delta_i) & \beta_i \cos(\delta_i) \end{matrix} \right]}$, where $\delta_i$ is the robot's orientation with $\dot{\delta}_i = [0, 1]\boldsymbol{u}_i$ and $\beta_i = \|\boldsymbol{v}_i\|$ in~\eqref{CBRS: eqn/sys}. 
The iLDAG $\mathcal{G}^{F}_{\mathrm{iLDAG}}$ (shown in Fig. \ref{CBRS:simulation_results}(a)) is constructed with branching factor $k = 3$ and $n_l = k = 3$ leaders (shown as red triangles in Fig. \ref{CBRS:simulation_results}). The leader set $\mathcal{V}_L = \{1, 2, 3\}$ forms a ring, with each leader anchoring a directed tree that is filled sequentially as per Definition~\ref{CBRS: def/iLDAG}. The resulting partition is: tree $\mathcal{T}_1$ with $V_2^{(1)} = \{4, 5, 6\}$ and $V_3^{(1)} = [7, 11]_{\mathbb{N}}$; tree $\mathcal{T}_2$ with $V_2^{(2)} = \{12, 13, 14\}$ and $V_3^{(2)} = [15, 19]_{\mathbb{N}}$; and tree $\mathcal{T}_3$ with $V_2^{(3)} = \{20, 21, 22\}$ and $V_3^{(3)} = [23, 27]_{\mathbb{N}}$, yielding $d = 3$ layers. Each tree is partially filled at the third layer, so $n$ does not saturate the full capacity $\mathcal{N}_3 = 3 \cdot (3^3 - 1)/2 = 39$. The geometric design parameters are: orbit radius $\eta_0 = 17$, safety radius $r^a_{im} = 1.5, \forall i,m \in \mathcal{V}$, and radial decay rate $\alpha = 0.436$, chosen to satisfy the bounds of Lemma \ref{CBRS: lem/safety}. The sensing range for layer 2 agents, shown as blue circles in Fig. \ref{CBRS:simulation_results}, $\lambda_i = 40, \forall i \in V_2^{(l)}, \forall l \in \mathcal{V}_{L}$, for layer 3 agents, shown as green circles in Fig. \ref{CBRS:simulation_results}, $\lambda_i = 18, \forall i \in V_3^{l}, \sigma_{pr^{(2)}(i)}(pr(i)) < k$, and for cross-tree links $\lambda_i = 50, \forall i \in V_3^{(l)}, \sigma_{pr^{(2)}(i)}(pr(i)) = k$ as per Lemma \ref{CBRS: lem/rigidity_st}. The consensus gains are set to $k_p = 4.0$ and $k_v = 6.0$ as per \eqref{eq: gains}. All agents are initialized at random positions within the sensing range of their respective tree parent.

\begin{figure}[!h]
\centering
\includegraphics[scale = 0.4]{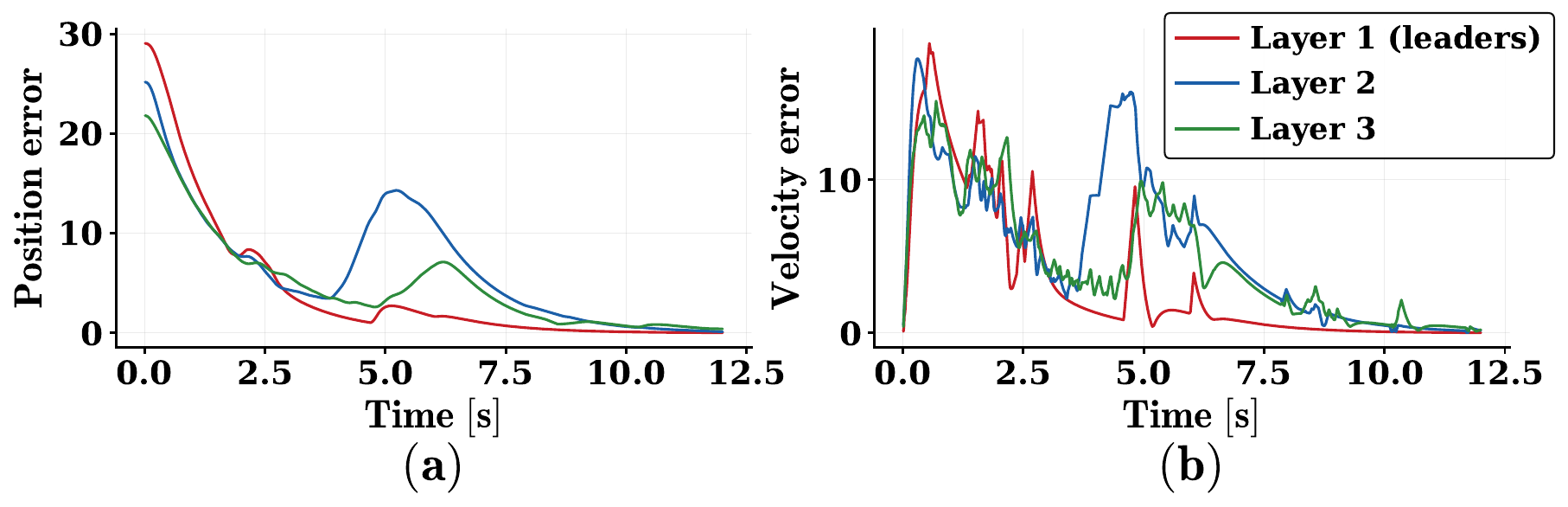}
\caption{Illustrations of splay formation tracking performance Layer-wise (a) mean position and (b) mean velocity error.}
\label{CBRS: fig/fig4_plots}
\end{figure}

The C3BF-QP safety filter (Theorem~\ref{thm:safe_splay_local_cbf}) enforces $\|\boldsymbol{p}_{im}(t)\| \geq r_{im}, \forall i \in \mathcal{V}, m \in \tilde{\mathcal{N}_i},$ by minimally modifying $\boldsymbol{u}_i^{nom}$ using the collision cone barrier $b(\boldsymbol{x}_{ij})$ in~\eqref{CBRS: eqn/c3bf}. As shown in Fig. \ref{CBRS:simulation_results}(b), during the transient phase ($t < 1$~s), the collision cone constraints activate between agents in close proximity (red wedges), steering relative velocities outside the cone while preserving convergence toward the splay scheme. By $t \approx 4$~s (cf. Fig. \ref{CBRS:simulation_results}(c)), the hierarchical recursive directed star structure emerges: followers in $V_2^{(l)}$ settle onto orbits of radius $\eta_1$ around their respective leaders, and $V_3^{(l)}$ agents orbit at radius $\eta_2 = \alpha\eta_1 \approx 3.2$ around their parents in $V_2^{(l)}$. By $t \approx 12$~s (Fig. \ref{CBRS:simulation_results}(d)), the splay scheme converges with siblings uniformly separated by $2\pi/k$ on each orbit. At the steady state, every follower $i$ has acquired a cross-family maintenance link to a second parent, satisfying $|\hat{\mathcal{E}}| = 2n - 3$ and the Laman conditions (Theorem~\ref{CBRS: thm/laman}).
The position error $\|p_i - p_j - p_{ij}^*\|$ and velocity error $\|v_i - v_j\|$ plots are shown in Fig.~\ref{CBRS: fig/fig4_plots}. No inter-agent collision occurs throughout the simulation, confirming forward invariance of the safe sets $\mathcal{C}_{ij}$ as guaranteed by Theorem~\ref{thm:safe_splay_local_cbf}. 


\begin{figure}[!h]
\centering
\includegraphics[scale = 0.5]{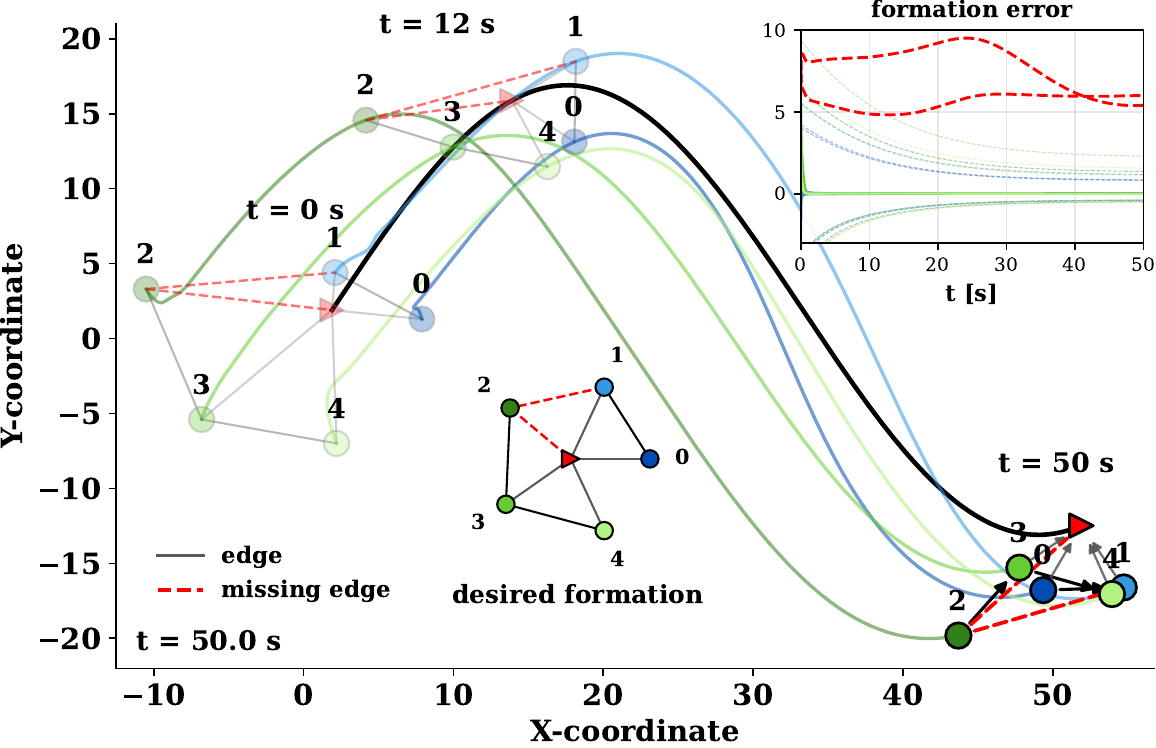}
\caption{Illustration of rigidity-matrix based methods for target-enclosing scenario with a non-rigid initial graph.}
\label{CBRS: fig/target_enc}
\end{figure}
\subsection{Comparison \& Discussion}
Most existing MRS approaches to formation control, tracking, and flocking assume the communication graph is already rigid and focus only on maintaining rigidity \cite{zhao2016localizability, pampatwar2021planar, rayabagi2024formation, zhao2019bearing}. In practice, under limited sensing range the graph is initially non-rigid, with links broken or created as robots move, so rigidity must first be \textit{acquired}. To the best of our knowledge, no existing distributed method acquires rigidity from a non-rigid graph while ensuring inter-agent safety and scalability.

Existing methods rest on two constructs: the rigidity matrix $R(\boldsymbol{p})$ \cite{zelazo2015decentralized, anderson2018rigid} and the graph Laplacian's algebraic connectivity $\rho_2$ \cite{capelli2020connectivity, ong2021network, bhatia2025decentralized}. To probe their behavior on an initially non-rigid graph, we run a target-enclosing experiment (Fig.~\ref{CBRS: fig/target_enc}) under a representative rigidity-maintenance controller \cite{zheng2026target}, where one agent's (shown as dark green in Fig. \ref{CBRS: fig/target_enc}) links to the target and its neighbors are broken due to limited sensing range (shown as red dashed lines in Fig. \ref{CBRS: fig/target_enc}). From an arbitrary $t=0$ configuration, agents with a requisite edge set settle onto the desired orbit by $t=12$, but the agent with the broken link (dark green) drifts and never converges by $t=50$, confirming that an initially non-rigid graph causes failure even under rigidity-maintenance control. We further investigate with two main questions.

\textit{1) Why is network rigidity important?} Rigidity ensures that the inter-agent distance constraints determine a unique formation shape \cite{anderson2003operations, eren2004rigid}, whereas mere connectivity does not, since the formation can deform while preserving connected tree edge distances (cf. Example~\ref{CBRS: ex/treenr}). Algebraically, $R(\boldsymbol{p})$ must have full rank $2n-3$ for the controller to converge \cite{zelazo2015decentralized, anderson2018rigid, zheng2026target}. A non-rigid graph leaves $R(\boldsymbol{p})$ rank-deficient, leaving deformation modes unconstrained and thereby preventing convergence to the desired formation (as seen in Fig.~\ref{CBRS: fig/target_enc}).

\textit{2) Do rigidity maintenance approaches suffice for rigidity acquisition?} No, neither rigidity nor connectivity maintenance acquires rigidity from an initial non-rigid graph. Rigidity-matrix methods maximize the rigidity eigenvalue \cite{zelazo2015decentralized}, presuming an already-rigid graph. If the rigidity eigenvalue starts at zero, the method cannot establish rigidity (cf. Fig.~\ref{CBRS: fig/target_enc}). Connectivity-maintenance methods \cite{capelli2020connectivity, ong2021network, bhatia2025decentralized} increase algebraic connectivity $\rho_2$, but connectivity is strictly weaker than rigidity as a connected graph can stay non-rigid, and increasing $\rho_2$ pushes the graph toward denser interaction graphs, clustering robots, and raising collision risk.

Both classes also suffer from deadlock for safe rigidity acquisition under a limited sensing range: an agent obstructed by neighbors under safety constraints halts before sensing range-based rigidity is achieved. The proposed splay scheme avoids this by giving each agent a distinct target within the recursive star structure, so no agent lingers near another, and the C3BF-QP controller formally guarantees rigidity acquisition via forward invariance of the pairwise safety sets (Theorem~\ref{thm:safe_splay_local_cbf}) from a non-rigid graph.

\section{Conclusion}
\label{CBRS: sec/conclude}
This work presents a distributed rigidity acquisition framework for heterogeneous nonlinear multi-agent systems under limited sensing, without assuming initial bearing rigidity or global position information. The proposed method guarantees rigid splay formation for network rigidity acquisition along with collision avoidance, validated through simulations and hardware experiments.


\appendix
\subsection{Proof of Theorem \ref{CBRS: thm/c3bf}}\label{CBRS: apx/A}
The function $b$ is a valid CBF for the system~\eqref{CBRS: eqn/sys}; see \cite{tayal2024collision}. Let us construct the standard CBF constraint $\dot{b} + \kappa(b) \geq 0$ as $\dot{b}= \langle \dot{\boldsymbol{p}}_{ij},\, \boldsymbol{v}_{ij} \rangle + \langle \boldsymbol{p}_{ij},\, \dot{\boldsymbol{v}}_{ij} \rangle - \|\boldsymbol{p}_{ij}\|\, \|\boldsymbol{v}_{ij}\|\,(\sin(\phi_{ij}))\dot{\phi}_{ij} +\cos(\phi_{ij})\Big[\frac{\boldsymbol{p}_{ij}^{\top}\dot{\boldsymbol{p}}_{ij}}{\|\boldsymbol{p}_{ij}\|}\, \|\boldsymbol{v}_{ij}\| + \frac{\boldsymbol{v}_{ij}^{\top}\dot{\boldsymbol{v}}_{ij}}{\|\boldsymbol{v}_{ij}\|}\, \|\boldsymbol{p}_{ij}\|\Big]  \geq - \kappa(b).$
We write $L_{h_{ij}} b\, \boldsymbol{q}_i$ as
\begin{align*}
    &L_{h_{ij}} b \, \boldsymbol{q}_i = {\Big[\boldsymbol{p}_{ij}^{\top} + \cos{\phi_{ij}\frac{\|\boldsymbol{p}_{ij}\|}{\|\boldsymbol{v}_{ij}\|}}\boldsymbol{v}_{ij}^{\top}\Big]}(\dot{\boldsymbol{v}}_j - \dot{\boldsymbol{v}}_i), \\&(\dot{\boldsymbol{v}}_j - \dot{\boldsymbol{v}}_i) = f_{v_j}(\boldsymbol{p}_j,\boldsymbol{v}_j) - f_{v_i}(\boldsymbol{p}_i,\boldsymbol{v}_i) \\&
    + h_{v_j}(\boldsymbol{p}_j,\boldsymbol{v}_j)\boldsymbol{u}_j - h_{v_i}(\boldsymbol{p}_i,\boldsymbol{v}_i)\boldsymbol{q}_i \\&
    = f_{v_j}(\boldsymbol{p}_j,\boldsymbol{v}_j) - f_{v_j}(\boldsymbol{p}_i,\boldsymbol{v}_i) + f_{v_j}(\boldsymbol{p}_i,\boldsymbol{v}_i) - f_{v_i}(\boldsymbol{p}_i,\boldsymbol{v}_i) \\&
    + h_{v_j}(\boldsymbol{p}_j,\boldsymbol{v}_j)\boldsymbol{u}_j - h_{v_i}(\boldsymbol{p}_j,\boldsymbol{v}_j)\boldsymbol{u}_j \\&
    +h_{v_i}(\boldsymbol{p}_j,\boldsymbol{v}_j)\boldsymbol{u}_j - h_{v_i}(\boldsymbol{p}_i,\boldsymbol{v}_i)\boldsymbol{u}_j \\&
    +h_{v_i}(\boldsymbol{p}_i,\boldsymbol{v}_i)\boldsymbol{u}_j - h_{v_i}(\boldsymbol{p}_i,\boldsymbol{v}_i)\boldsymbol{q}_i.
\end{align*}
Using Assumption \ref{CBRS: assum/A3} and the triangle inequality, we obtain $\|\boldsymbol{p}_i\| \leq \mathcal{B}_p(i,pr(i)), \|\boldsymbol{v}_i\| \leq \mathcal{B}_v(i,pr(i))$. 
Using Lipschitz continuity, Assumption \ref{CBRS: assum/A2}, \ref{CBRS: assum/A4}, minimum eigenvalue $\mu_{\min}(h_{v_i}g_i) \geq (\Gamma_{ig}(\mathcal{B}_p(i,pr(i)) + \mathcal{B}_v(i,pr(i))))^{-1}\Gamma_i$, $\boldsymbol{q}_i = g_i(\delta_i)\boldsymbol{\xi}_{ij}\boldsymbol{\xi}_{ij}^{\top}\hat{\boldsymbol{q}}_i$, we have $L_{h_{ij}} b \, \boldsymbol{q}_i \geq -\Gamma_{jf}\|\boldsymbol{x}_{ij}\|\|\boldsymbol{\xi}_{ij}\| - \bar{\Delta}\|\boldsymbol{\xi}_{ij}\| 
   - \bar{\epsilon}\|\boldsymbol{u}_j\|\|\boldsymbol{\xi}_{ij}\| 
    - \Gamma_{ih}\|\boldsymbol{u}_j\|\|\boldsymbol{x}_{ij}\|\|\boldsymbol{\xi}_{ij}\|
    - (\Gamma_{ih}+1)\|\boldsymbol{u}_j\|(\mathcal{B}_p(i,pr(i)) + \mathcal{B}_v(i,pr(i)))\|\boldsymbol{\xi}_{ij}\|
    - (\Gamma_{ig}(\mathcal{B}_p(i,pr(i)) + \mathcal{B}_v(i,pr(i))))^{-1}\Gamma_i\|\boldsymbol{\xi}_{ij}\|^2\boldsymbol{\xi}_{ij}^{\top}\hat{\boldsymbol{q}}_i.$
Thus, given that we obtain $\|\boldsymbol{u}_j\|$ by Assumption \ref{CBRS: assum/A1}, instead of solving for the constraint $\dot{b} + \kappa(b) \geq 0$, we use its lower bound $L_{f_{ij}}' b(\boldsymbol{x}_{ij}) + L_{h_{ij}}' b(\boldsymbol{x}_{ij})\, \hat{\boldsymbol{q}}_i + \kappa(b(\boldsymbol{x}_{ij})) \geq 0$, where $L_{f_{ij}}' b(\boldsymbol{x}_{ij}) = L_{f_{ij}} b   -\Gamma_{jf}\|\boldsymbol{x}_{ij}\|\|\boldsymbol{\xi}_{ij}\| - \bar{\Delta}\|\boldsymbol{\xi}_{ij}\| - \bar{\epsilon}\|\boldsymbol{u}_j\|\|\boldsymbol{\xi}_{ij}\| - \Gamma_{ih}\|\boldsymbol{u}_j\|\|\boldsymbol{x}_{ij}\|\|\boldsymbol{\xi}_{ij}\| - (\Gamma_{ih}+1)\|\boldsymbol{u}_j\|(\mathcal{B}_p(i,pr(i)) + \mathcal{B}_v(i,pr(i)))\|\boldsymbol{\xi}_{ij}\|$ and $L_{h_{ij}}' b\, \hat{\boldsymbol{q}}_i = - (\Gamma_{ig}(\mathcal{B}_p(i,pr(i)) + \mathcal{B}_v(i,pr(i))))^{-1}\Gamma_i\|\boldsymbol{\xi}_{ij}\|^2\boldsymbol{\xi}_{ij}^{\top}\hat{\boldsymbol{q}}_i.$ The constraint is feasible as $\Gamma_{ig}(\mathcal{B}_p(i,pr(i)) + \mathcal{B}_v(i,pr(i)))\|\boldsymbol{\xi}_{ij}\|^2 \neq 0$, since $\Gamma_{ig}, \mathcal{B}_p, \mathcal{B}_v, \|\boldsymbol{\xi}_{ij}\| \in \mathbb{R}^{+}$. With the lower bound enforced to be positive, the forward invariance result follows from \cite[Theorem~1]{tayal2024collision}. This concludes the proof. 

\subsection{Proof of Lemma \ref{CBRS: thm/vertex_addition}}\label{CBRS: apx/B}
Here, each vertex considered as per index ordering, contributes exactly two incident maintenance links, forming a new triplet structure $(i,j_1,j_2)$ on the initial rigid graph that satisfies condition (2) of Theorem~\ref{CBRS: thm/laman}, i.e., the number of edges in the triplet, $|\hat{\mathcal{E}}''| = 2 \leq 6 - 3 = 3$. The total number of edges increases by two from the initially rigid graph (starting from the rigid subgraph of leaders), resulting in $2n-3 + 2 = 2(n+1) - 3$ edges, which satisfies condition (1) of Theorem~\ref{CBRS: thm/laman}. This completes the proof.

\subsection{Proof of Lemma \ref{CBRS: lem/safety}}\label{CBRS: apx/C}
For safety within each family as in each local star formed by agents $i, m \in S_{pr(i)}(t)$, parent-children $(i,pr(i))$ separation by $\eta_{s-1}$ should satisfy $\eta_{s-1} \geq r^a$ and for adjacent siblings $i, m,$ on orbit $\eta_{s-1}$ separated by $\|\boldsymbol{p}_{im}\| = 2\eta_{s-1} \sin\left(\frac{\pi}{k}\right),$
should satisfy $\|\boldsymbol{p}_{im}\| \geq r^a$. Both inequalitites yield the first two terms in the first inequality. For cross-family safety in the same tree $\mathcal{T}_l$, consider tree parent $pr(i) = j_1\in V_{s-1}^{(l)}$, and another agent $j_2\in V_{s-1}^{(l)}$, sharing a common grandparent, i.e.,  $j_1, j_2 \in S_{pr(j_1)}$, separated by $\|\boldsymbol{p}_{j_1 j_2}\| = 2\eta_{s-2}\sin(\pi/k)$. The cumulative radial intrusion from the children orbit of $i \in S_{j_1}$ and $m \in S_{j_2}$ is $2\eta_{s-1}$. For inter-agent safety for pair $(i, m)$, $\|\boldsymbol{p}_{im}\| \geq \|\boldsymbol{p}_{j_1 j_2}\| - 2\eta_{s-1} =2\eta_{s-2}\sin(\pi/k) - 2\eta_{s-1} > r^a.$
Lower bounding this strictly by $r^a$ yields the third term in the first inequality.
For cross-family safety but across adjacent trees, the cumulative radial intrusion between two adjacent trees ($l, l'$) can be calculated as (geometric series sum) $\eta_c^{ll'} = 2\sum_{m=1}^{d-1} \eta_m$. So, for inter-agent safety between any $i \in V_{\mathcal{T}_l}$ and $m \in V_{\mathcal{T}_{l'}}$, $r_{im} \leq 2\eta_0sin(\frac{\pi}{n_l}) - \eta_c^{ll'}$. Lower bounding this strictly by $r^a$, and simplifying $\eta_c^{ll'} = \frac{2(\eta_0\alpha - \eta_0\alpha^d)}{1 - \alpha}$ yields the last inequality. This completes the proof. 

\subsection{Proof of Lemma \ref{CBRS: lem/rigidity_st}}\label{CBRS: apx/D}
For the child $i$ in the splay scheme connected to the parent $j_1 \in V_{s-1}^{(l)}$, using $\eta_{s-1} \leq \lambda_i$, the nearest adjacent parent $j_2 \in V_{s-1}^{(l)}$ is at $\|\boldsymbol{p}_{j_1 j_2}\| = 2\eta_{s-2}\sin(\pi/k)$ from $j_1$. By the triangle inequality, the distance from $i$ at radius $\eta_{s-1} = \alpha\eta_{s-2}$ to $j_2$ is bounded by
$\|\boldsymbol{p}_{i j_2}\| \leq \|\boldsymbol{p}_{j_1 j_2}\| + \|\boldsymbol{p}_{i j_1}\| 
        = 2\eta_{s-2}\sin\left(\frac{\pi}{k}\right) + \alpha\eta_{s-2}.$
Upper bounding the required distance by $\lambda_i$ gives the inequality and forms the edge as per edge addition sequence in Lemma \ref{CBRS: thm/vertex_addition} for preserving rigidity.
Choosing the agent's sensor such that $\lambda_i$ is greater than or equal to the bound in \eqref{eq:rigidity_st} guarantees detection of the second parent, thereby establishing cross family links in the same tree. For the case when the link is across adjacent trees is discussed next.
Similarly, with connected tree parent $j_1$, using $\eta_{s-1} \leq \lambda_i$, we can write $\|\boldsymbol{p}_{i j_2}\| \leq \|\boldsymbol{p}_{ij_1}\| + \|\boldsymbol{p}_{j_1l}\| + \|\boldsymbol{p}_{ll'}\| + \|\boldsymbol{p}_{l'j_2}\|
        \leq \eta_{s-1} + \sum_{r=1}^{s-2}\eta_r + 2\eta_0\sin(\frac{\pi}{n_l}) + \sum_{r=1}^{s-2}\eta_r.$
Upper bounding the required distance by $\lambda_i$ gives the inequality and follows Lemma \ref{CBRS: thm/vertex_addition} for rigidity.

\subsection{Proof of Theorem \ref{CBRS: thm/hetero_splay}}\label{CBRS: apx/E}
Consider $\boldsymbol{e}_{ip}=\sum_{j \in \mathcal{S}_i \cup \{pr(i)\}} a_{ij}(\boldsymbol{p}_i - \boldsymbol{p}_j - \boldsymbol{p}_{ij}^*)$$, \boldsymbol{e}_{iv} = \sum_{j \in \mathcal{S}_i \cup \{pr(i)\}} a_{ij}
(\boldsymbol{v}_i - \boldsymbol{v}_j)$$, \boldsymbol{y}_i = [\boldsymbol{e}_{ip}^{\top},\ \boldsymbol{e}_{iv}^{\top}]^{\top}$, and stacking $\boldsymbol{e}_{ip}, \boldsymbol{e}_{iv} \in \mathbb{R}^2$ for all followers results in $\boldsymbol{e}_{p}, \boldsymbol{e}_{v} \in \mathbb{R}^{2n}, \boldsymbol{y} = [\boldsymbol{e}_{p}^{\top},\ \boldsymbol{e}_{v}^{\top}]^{\top}$. Then we can write the error dynamics as $\dot{\boldsymbol{y}}=(\hat{\mathcal{L}} \otimes I_2) {\boldsymbol{y}} + \boldsymbol{d}({\boldsymbol{y}}, t)$, where $\hat{\mathcal{L}} = \begin{bmatrix}
    0 & I \\
    -k_p \mathcal{L}^{F}_r & -k_v \mathcal{L}^{F}_r
\end{bmatrix}$ and non-linear dynamics terms are considered as disturbance $\boldsymbol{d}({\boldsymbol{y}}, t)$. The matrix $\hat{\mathcal{L}}$ has zero eigenvalue of multiplicity two and all other eigenvalues of negative real parts due to the presence of a directed spanning tree in $\mathcal{G}^{F}_r$, following \eqref{eq: gains}, as per the results in \cite[Theorem 1]{yu2010some}.
Using Lipschitz constants and conditions \eqref{eq:flipsch_c}, \eqref{eq:hlipsch_c}, we can bound disturbance $\| \boldsymbol{d}({\boldsymbol{y}_i}, t) \| \leq \sum_{j \in \mathcal{S}_i \cup \{pr(i)\}} a_{ij} \|\boldsymbol{d}({\boldsymbol{y}_{ij}}, t)\|, \boldsymbol{y}_{ij} = [(\boldsymbol{p}_i - \boldsymbol{p}_j - \boldsymbol{p}_{ij}^*)^{\top}, (\boldsymbol{v}_i - \boldsymbol{v}_j)^{\top}]^{\top}$ where $\| \boldsymbol{d}({\boldsymbol{y}_{ij}}, t) \| = \|f_{v_i}(\boldsymbol{p}_i,\boldsymbol{v}_i) - f_{v_j}(\boldsymbol{p}_j,\boldsymbol{v}_j)\| 
    + \|(h_{v_i}(\boldsymbol{p}_i,\boldsymbol{v}_i) - I_2)\boldsymbol{u}_i - (h_{v_j}(\boldsymbol{p}_j,\boldsymbol{v}_j) - I_2)\boldsymbol{u}_j\|
    \leq \|f_{v_i}(\boldsymbol{p}_i,\boldsymbol{v}_i) - f_{v_i}(\boldsymbol{p}_j,\boldsymbol{v}_j)\| + \|f_{v_i}(\boldsymbol{p}_j,\boldsymbol{v}_j) - f_{v_j}(\boldsymbol{p}_j,\boldsymbol{v}_j)\| 
    + \|(h_{v_i}(\boldsymbol{p}_i,\boldsymbol{v}_i) - h_{v_j}(\boldsymbol{p}_i,\boldsymbol{v}_i))\boldsymbol{u}_i \|
    + \|(h_{v_j}(\boldsymbol{p}_i,\boldsymbol{v}_i) - h_{v_j}(\boldsymbol{p}_j,\boldsymbol{v}_j))\boldsymbol{u}_i \|
    +\|(h_{v_j}(\boldsymbol{p}_j,\boldsymbol{v}_j) - I_2)(\boldsymbol{u}_i - \boldsymbol{u}_j)\| \leq {\gamma}_{ij}\|\boldsymbol{y}_{ij}\| + {\Delta}_{ij},$
where ${\gamma}_{ij} = \Gamma_{if} + \sum_{m \in \mathcal{S}_i \cup \{pr(i)\}} a_{im}(\Gamma_{mh}\|\boldsymbol{x}_{im}\|+ 2\Gamma_{mh}(\mathcal{B}_p(m,i) + \mathcal{B}_v(m,i)) + \bar{\epsilon})(k_p + k_v), {\Delta}_{ij} = \Gamma_{if} \|\boldsymbol{p}_{ij}^*\| + \bar{\Delta}$. Similarly, we can write $\| \boldsymbol{d}({\boldsymbol{y}}, t) \| \leq {\gamma}^*\|\boldsymbol{y}\| + {\Delta}$, where $\gamma^* \geq \gamma = \max_i (\max_{(i,j)}\gamma_{ij}), \Delta = \max_i (\max_{(i,j)}\Delta_{ij})$.\\
The off-diagonal terms $[\mathcal{L}^{F}_r]_{ij} = -\tau \frac{\lambda_i}{\lambda_j}, i \neq j$ provides the property that $[\mathcal{L}^{F}_r]_{ij}\lambda_j^2 = [\mathcal{L}^{F}_r]_{ji}\lambda_i^2$. Due to which under similarity transformation $T = diag(\lambda_1, \cdots, \lambda_n), \tilde{\mathcal{L}} := T^{-1}\mathcal{L}^F_rT,$ is a symmetric matrix, and we can conclude that $\|e^{\mathcal{L}^F_rt}\| = \|T\|\|e^{\tilde{\mathcal{L}}t}\| \|T^{-1}\| \leq C = \frac{\lambda_{\max}}{\lambda_{\min}}$. So, $\|e^{(\hat{\mathcal{L}} \otimes I_2)t}\boldsymbol{y}\| \leq \|e^{\hat{\mathcal{L}}t}\|\|\boldsymbol{y}\| \leq \frac{\lambda_{\max}}{\lambda_{\min}} e^{-{\mu^*}t}\|\boldsymbol{y}\|, \mu^* = \frac{k_v}{2}\min_{\mu_z\neq 0} \operatorname{Re}(\mu_{z}(\mathcal{L}^{F}_r))$, following equation (6) in \cite{yu2010some}. Using variation of constants formula, we can write $\boldsymbol{y}(t) = e^{(\hat{\mathcal{L}} \otimes I_2)t}\boldsymbol{y}(0) + \int_0^t e^{(\hat{\mathcal{L}} \otimes I_2)(t-s)} \boldsymbol{d}(\boldsymbol{y}(s), s)\, \mathrm{d}s, \|\boldsymbol{y}(t)\| \leq C\, e^{-\mu^* t}\|\boldsymbol{y}(0)\| + C \int_0^t e^{-\mu^*(t-s)} \left({\gamma}^*\|\boldsymbol{y}(s)\| + {\Delta}\right) \mathrm{d}s, \|\boldsymbol{y}(t)\| e^{\mu^* t}\leq C\, \|\boldsymbol{y}(0)\|+ C \int_0^t e^{\mu^*s} \left({\gamma}^*\|\boldsymbol{y}(s)\| + {\Delta}\right) \mathrm{d}s.$
Using Gronwall's inequality, we obtain $\|\boldsymbol{y}(t)\|
\le
C \|\boldsymbol{y}(0)\|
e^{-(\mu^{*}-C\gamma^*)t}
+
\frac{C\Delta}{\mu^{*} - C\gamma^*}
(1 - e^{-(\mu^* - C\gamma^*) t})$. Therefore, using \eqref{eq: compn}, $\lim_{t \rightarrow \infty} \|\boldsymbol{y}(t)\| \leq \frac{C\Delta}{\mu^{*} - C\gamma^*}$, and second-order consensus is achieved with worst-case tracking error as $\frac{C\Delta}{\mu^{*} - C\gamma^*}$. Let $\zeta := \mu^{*} - C\gamma^* > 0$. The Gronwall bound gives $\|\boldsymbol{y}(t)\| \leq \frac{C\Delta}{\eta} + \left(C\|\boldsymbol{y}(0)\| - \frac{C\Delta}{\zeta}\right) e^{-\zeta t}.$
Thus, for any $\varepsilon > 0$, $\|\boldsymbol{y}(t)\| \leq \frac{C\Delta}{\eta} + \varepsilon$, for all $t \geq T(\varepsilon)$, where $T(\varepsilon) := \frac{1}{\zeta}  \ln\!\left( \frac{C\|\boldsymbol{y}(0)\| - \frac{C\Delta}{\zeta}}{\varepsilon} \right)$.

\bibliographystyle{unsrtnat}
\bibliography{sources_arxiv}  






\end{document}